\documentclass[11pt]{article}
\usepackage[final]{acl}

\usepackage{times}
\usepackage{latexsym}

\usepackage[T1]{fontenc}

\usepackage[utf8]{inputenc}

\usepackage{microtype}

\usepackage{inconsolata}

\usepackage{graphicx}
\usepackage{subcaption}

\usepackage{amsfonts}
\usepackage{amsmath}
\usepackage{amsthm}
\newtheorem{theorem}{Theorem}

\usepackage{algorithm}
\usepackage{algorithmic}
\usepackage{amssymb}
\usepackage{booktabs}
\usepackage[table]{xcolor}
\usepackage{multirow}
\usepackage{arydshln}
\usepackage{hyperref}
\usepackage{stfloats}
\title{Difficulty-Adaptive Tree-Structured Policy Optimization\\
for Expanding Reasoning Coverage in RLVR}

\author{
  Youngjun Yu \qquad
  Sanghwan Jang \qquad
  Hwanjo Yu\thanks{Corresponding author.} \\
  Pohang University of Science and Technology (POSTECH) \\
  \texttt{\{colin31472, s.jang, hwanjoyu\}@postech.ac.kr}
}

\begin{document}
\maketitle
\begin{abstract}
Reinforcement Learning with Verifiable Rewards (RLVR) has been central to the recent success of Large Reasoning Models. However, while RLVR significantly improves single-sample accuracy, it often fails to expand the model's intrinsic reasoning coverage (pass@$k$) due to limited exploration during training. To address this, we optimize the structural design of train-time rollouts to enhance pass@$k$. Our analysis identifies three key design principles: (1) difficulty-adaptive rollout can play an important role in expanding pass@k, beyond serving as an efficiency heuristic; (2) tree-based rollout outperforms parallel sampling in discovering correct answers; and (3) sentence-entropy-guided forking overcomes the \textit{localization} phenomenon of token-level branching to maximize semantic diversity. Building on these insights, we propose \textbf{DATPO} (\textbf{D}ifficulty-\textbf{A}daptive Sentence-entropy-guided \textbf{T}ree-structured \textbf{P}olicy \textbf{O}ptimization). DATPO integrates difficulty-adaptive tree search with a \textit{sibling-diversity} advantage term, explicitly promoting semantic diversity to expand reasoning coverage during training. Experiments on mathematical reasoning benchmarks demonstrate that DATPO outperforms baselines especially in pass@$k$, which directly translates to superior test-time scaling performance.\footnote{Code: \url{https://github.com/colin31472/DATPO}}
\end{abstract}

\section{Introduction}
The paradigm of Large Reasoning Models (LRMs) has achieved remarkable success in eliciting rigorous problem-solving capabilities from foundation models \citep{jaech2024openai, kimiteam2025kimik15scalingreinforcement}. A key mechanism behind this is Reinforcement Learning with Verifiable Rewards (RLVR), exemplified by DeepSeek-R1 \citep{guo2025deepseek} through its successful application of the Group Relative Policy Optimization (GRPO) \citep{shao2024deepseekmath}. 

In parallel, LRM research has increasingly emphasized test-time scaling strategies, including CoT \citep{wei2022chain}, majority voting \citep{wang2022self}, best-of-N \citep{stiennon2020learning}, and Monte Carlo Tree Search (MCTS) \citep{ha-etal-2025-dsg}. Fundamentally, the effectiveness of these methods is limited by the model's intrinsic reasoning coverage---the breadth of valid reasoning paths the model can explore, typically measured by pass@$k$. Without sufficient reasoning coverage, even large candidate sets are unlikely to include a correct reasoning path, causing test-time scaling to aggregate or select among recurring errors \citep{brown2024large, zhao2025sample}.

However, recent studies \citep{yue2025does, dang2025assessing, wu2025invisible} argue that RLVR often fails to unlock new reasoning capabilities beyond the reasoning coverage of the base model. This limitation stems from a lack of explicit exploration strategies that steer the learning or sampling process toward underexplored regions. Without such guidance, the model merely exploits known solutions rather than discovering novel ones.

While recent studies introduce exploration strategies to expand reasoning coverage \citep{walder2025pass, yao2025diversity, wang2025beyond}, they largely focus on optimization-level interventions, leaving the train-time rollout as a fixed parallel structure. Tree-based frameworks \citep{hou2025treerl, liu2025attention} have begun to move beyond standard parallel sampling. However, these methods primarily utilize tree structures for credit assignment or computational efficiency. Consequently, it remains underexplored which structural choices in train-time rollouts actually expand the model's reasoning coverage.

To bridge this gap, we conduct a systematic empirical analysis of train-time rollouts, revealing three core design principles to maximize reasoning coverage, extending beyond improvements in single-sample accuracy. First, \textbf{difficulty-adaptive rollout is an important factor in expanding pass@k, rather than merely an efficiency heuristic.} While increasing the rollout budget improves pass@$k$ when training on hard problems, we observe that it can actually degrade pass@$k$ when training on easy problems. Second, \textbf{tree-based rollout outperforms parallel sampling} in discovering correct answers by efficiently utilizing limited computational budgets. Finally, we observe that conventional token-guided forking suffers from \textit{localization}, where high-entropy tokens cluster in narrow segments and trap exploration. To overcome this, \textbf{forking decisions must be elevated to a broader semantic level} to discover genuinely diverse reasoning paths.

Building on these insights, we propose \textbf{DATPO} (\textbf{D}ifficulty-\textbf{A}daptive Sentence-entropy-guided \textbf{T}ree-structured \textbf{P}olicy \textbf{O}ptimization). For train-time exploration, DATPO combines the structural advantages of tree-based search with difficulty-adaptive rollout to expand reasoning coverage. Forking points are selected using sentence-level entropy signals, broadening the granularity of uncertainty estimation to avoid \textit{localization}. Furthermore, to maximize coverage-expanding benefits, we augment the advantage function with an annealed \textit{sibling-diversity} term. This explicitly rewards semantically diverse reasoning branches, encouraging the model to explore a wider range of valid reasoning paths.

Extensive experiments on mathematical reasoning benchmarks demonstrate the effectiveness of DATPO, which achieves the highest average pass@$k$ among the compared methods. Furthermore, this expanded reasoning coverage directly translates to superior test-time scaling with majority voting (maj@$k$).

\begin{figure*}[t]
    \centering
    \begin{subfigure}[b]{0.32\textwidth}
        \centering
        \includegraphics[width=\linewidth]{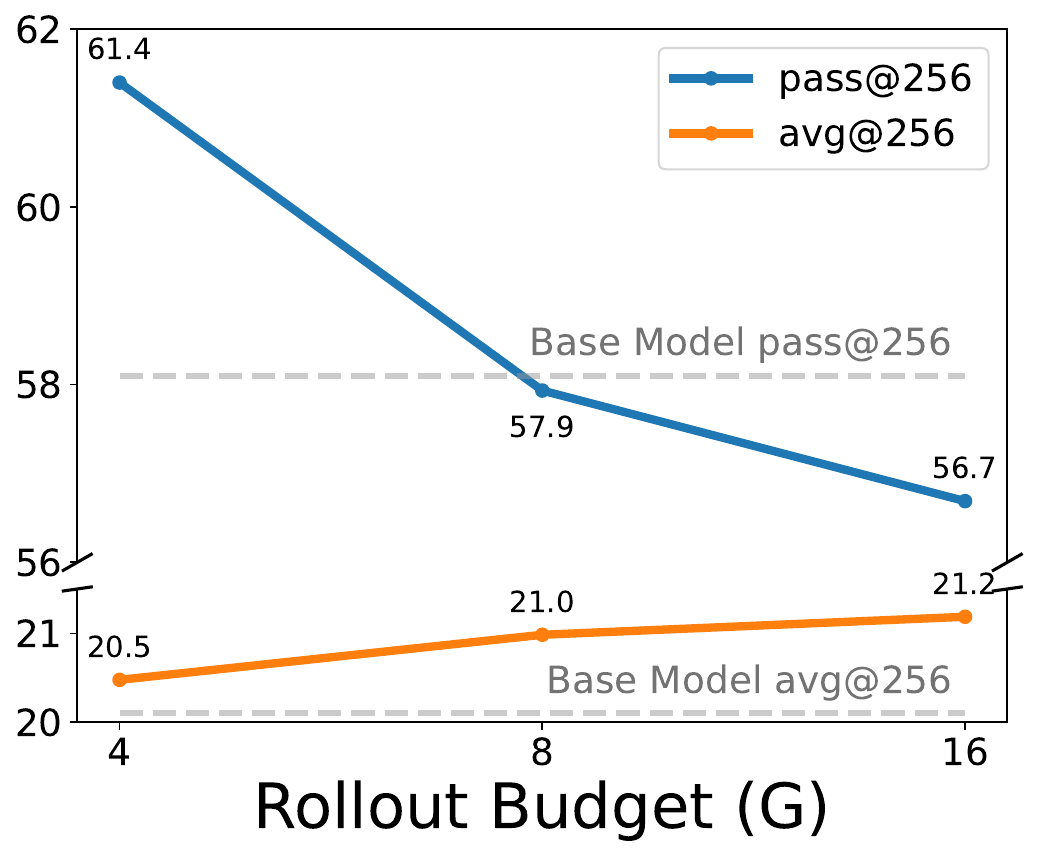}
        \caption{Models trained on \textbf{Easy} dataset}
        \label{fig:easy_res}
    \end{subfigure}
    \hfill
    \begin{subfigure}[b]{0.32\textwidth}
        \centering
        \includegraphics[width=\linewidth]{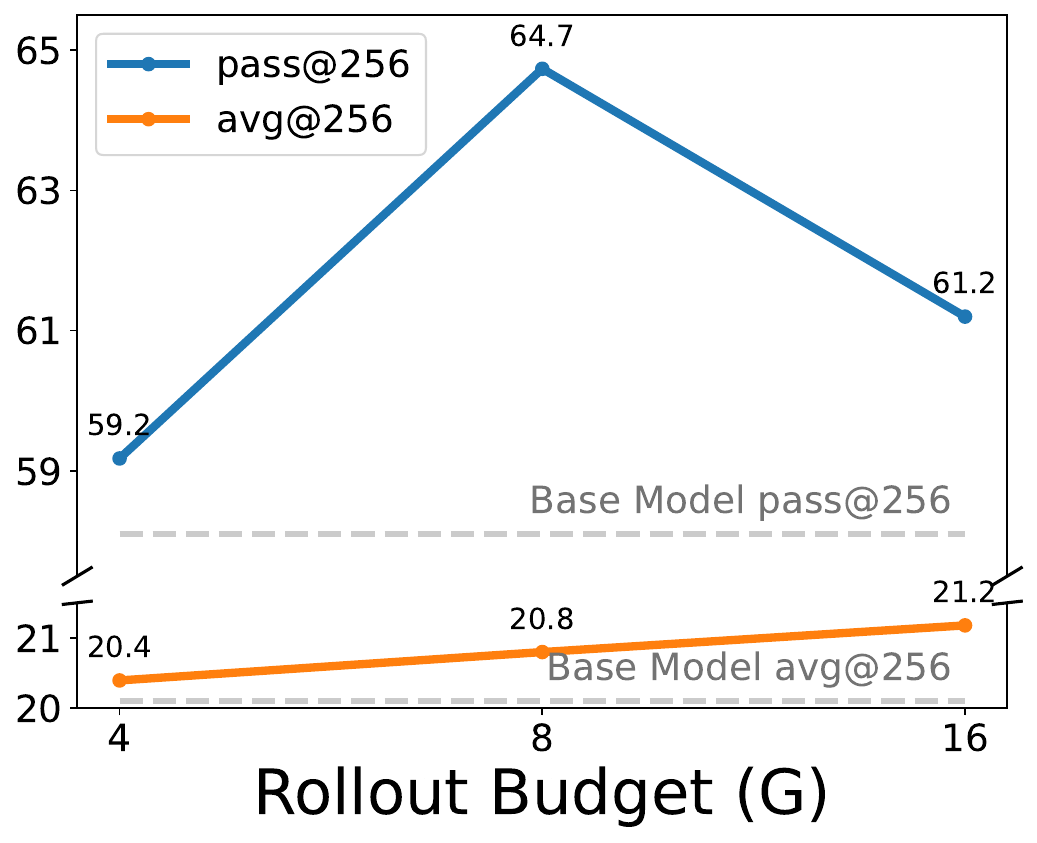} 
        \caption{Models trained on \textbf{Medium} dataset}
        \label{fig:med_res}
    \end{subfigure}
    \hfill
    \begin{subfigure}[b]{0.32\textwidth}
        \centering
        \includegraphics[width=\linewidth]{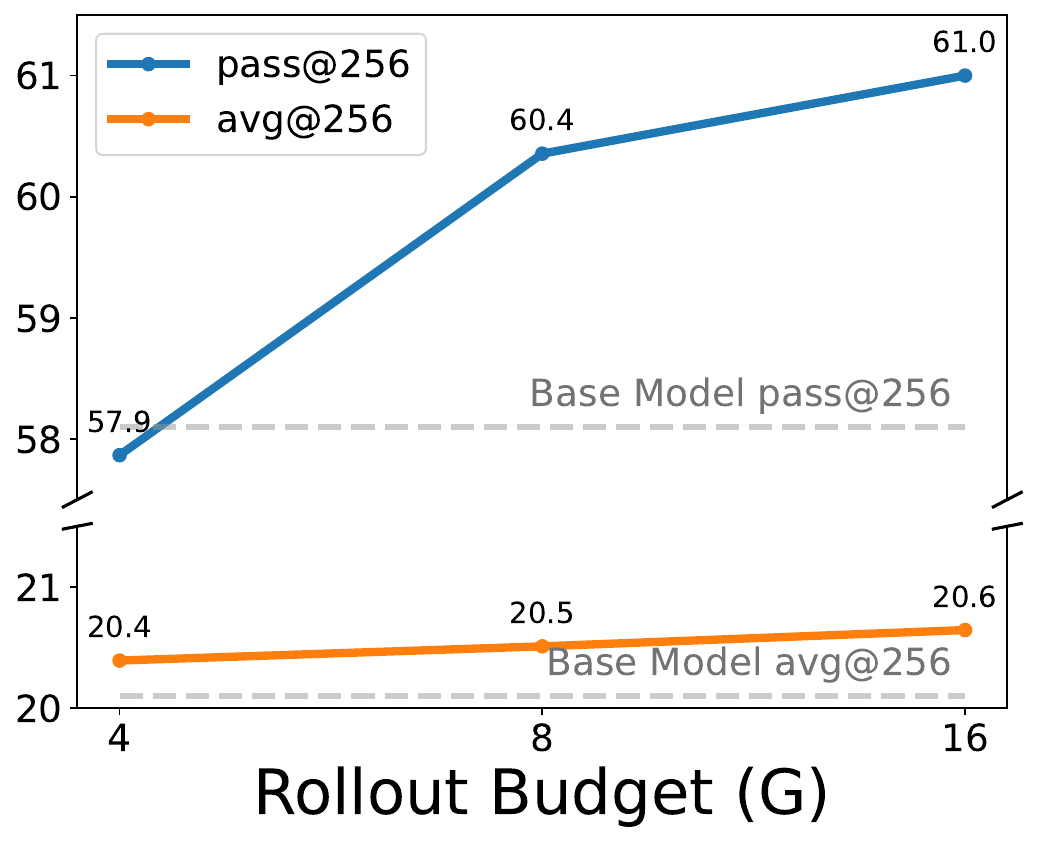}
        \caption{Models trained on \textbf{Hard} dataset}
        \label{fig:hard_res}
    \end{subfigure}
    
    \caption{avg@256 and pass@256 performance across different rollout budgets for models trained on different difficulty subsets.}
    \label{fig:adaptive_results}
\end{figure*}

In summary, our main contributions are as follows:
\begin{itemize}
    \item We reveal key structural principles for expanding reasoning coverage: difficulty-adaptive rollout is a key factor, and tree-structured rollouts inherently outperform parallel sampling, with sentence-entropy forking further amplifying this structural advantage.
    \item We propose DATPO, a novel RLVR framework that integrates these structural insights with a diversity-augmented advantage to maximize coverage-expanding benefits. Experiments on mathematical reasoning benchmarks validate our approach, showing notable improvements in pass@$k$.
\end{itemize}




\section{Analyzing the Structural Design of Train-time Rollouts}
\label{sec:3_analysis}
Despite extensive research on expanding reasoning coverage, the structural design of \textit{train-time} rollouts remains underexplored. This section investigates optimal rollout structures to maximize not only single-sample accuracy (avg@$k$) but also broader reasoning coverage (pass@$k$). Our analysis is driven by three primary research questions: (1) \textbf{Difficulty-Adaptive Allocation:} How should compute budgets be allocated across problems of varying difficulty to expand reasoning coverage?; (2) \textbf{Topology of Search Space:} Which search space structure (tree-based or parallel rollout) more efficiently discovers correct reasoning trajectories?; and (3) \textbf{Optimal Forking Strategy:} Where should forking (branching) occur in tree-based rollouts to effectively promote diverse reasoning paths?

\subsection{Necessity of Difficulty-Adaptive Rollout}
\label{sec:adaptive_rollout}
Standard RLVR algorithms, notably GRPO \citep{shao2024deepseekmath} and its variants \citep{yu2025dapo, liu2025understanding}, typically generate a uniform number of rollouts regardless of problem difficulty. While existing difficulty-adaptive rollout strategies prioritize \textit{computational efficiency} \citep{liu2025attention, liao-etal-2025-enhancing, zhang2025improving, zheng2025act, li2025knapsack}, we challenge this perspective by demonstrating that adaptive allocation is not merely a resource-saving heuristic but a crucial factor in expanding reasoning coverage.

\subsubsection{Empirical Analysis}
\label{sec:difficulty_adaptive_analysis}
We first partitioned the training dataset into Easy, Medium, and Hard subsets based on the accuracy of the base model. Using GRPO \citep{shao2024deepseekmath}, we conduct a total of 9 training runs. For each difficulty subset, we vary the rollout budget by adjusting the GRPO group size, $G \in \{4, 8, 16\}$. The resulting 9 models are evaluated on 3 different benchmarks, reporting both avg@256 and pass@256 metrics (detailed in \autoref{app:experimental_details_adaptive}).

The aggregated results, as illustrated in \autoref{fig:adaptive_results}, reveal two key insights. First, \textbf{avg@256 shows consistent, yet marginal, improvements with increased rollout budgets ($G$) across all difficulties.} We attribute this to more stable exploitation; larger GRPO group sizes reduce variance in advantage estimation and yield more frequent meaningful learning signals, effectively solidifying the policy distribution around correct solutions.

Second, \textbf{the impact of increasing rollouts on pass@256 shifts from detrimental to beneficial depending on the difficulty of the training dataset.} For models trained on the \textit{Easy} dataset (\autoref{fig:easy_res}), larger budgets degrade pass@256. This occurs because the model over-exploits specific, easy templates, overfitting to familiar problems while failing entirely on harder ones. Since pass@256 requires only a single correct path, generating redundant successes for easy problems provides no coverage benefit. Consequently, despite marginal gains in avg@256, pass@256 decreases as the rollout budget increases. Conversely, for models trained on the \textit{Hard} dataset (\autoref{fig:hard_res}), where valid paths are scarce, larger budgets promote broader exploration without overfitting to narrow patterns, consistently improving pass@256. Between these opposing behaviors, models trained on the \textit{Medium} dataset (\autoref{fig:med_res}) exhibit a non-monotonic trend, peaking at $G=8$ by effectively balancing exploration and over-exploitation.

These findings suggest that difficulty-adaptive rollout is not merely a heuristic for compute efficiency, but a strategic necessity for balancing the exploitation-exploration trade-off and maximizing the model's reasoning coverage. Comprehensive theoretical analysis is provided in \autoref{app:adaptive_theory}.

\subsection{Tree vs. Parallel Rollout}
\label{sec:tree_vs_parallel}
In RLVR, the primary challenge in expanding reasoning coverage is the scarcity of positive learning signals on complex problems. These signals emerge only when the model discovers a correct reasoning path. To address this, we investigate the properties of train-time rollout structures by comparing the conventional \textit{parallel rollout} with a \textit{tree-structured rollout} strategy, aiming to determine which topology more efficiently discovers correct trajectories.

\begin{figure}[t]
    \centering
    \includegraphics[width=0.8\linewidth]{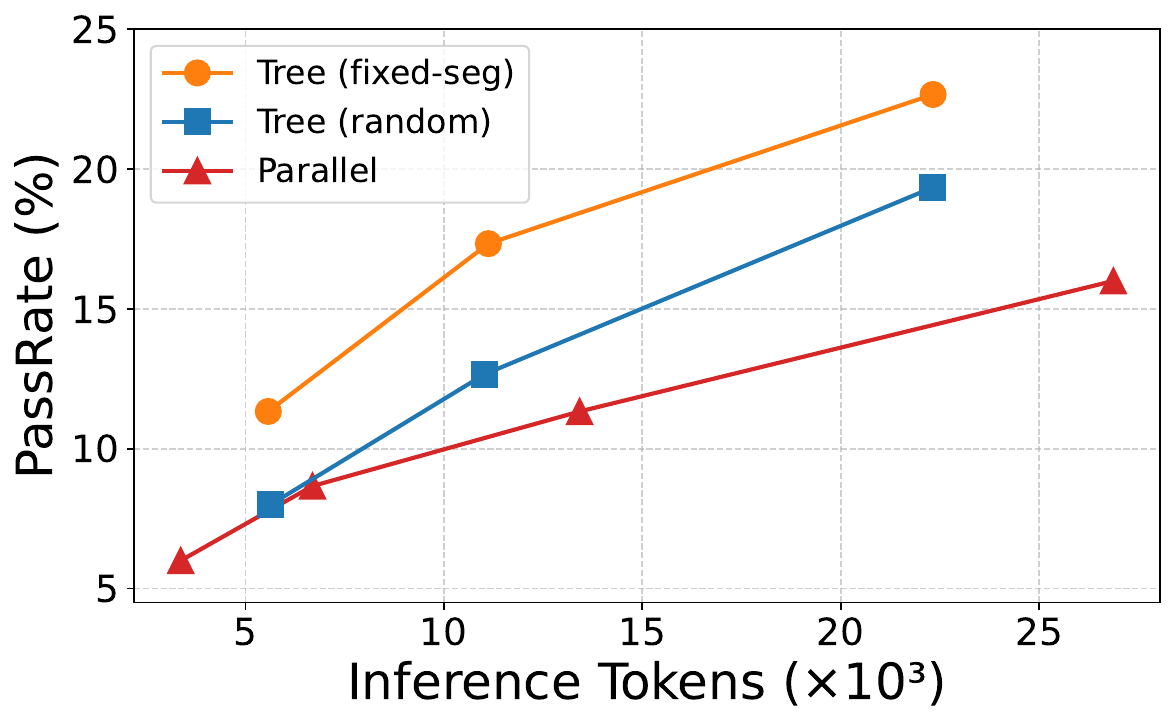}
    \caption{Comparison of PassRate against inference token consumption between parallel and two tree-structured sampling using different forking strategies.}
    \label{fig:tree_parallel}
\end{figure}

\subsubsection{Tree Rollout Algorithm}
\label{sec:tree_rollout}
We employ a generalized and simplified variant of the two-phase tree rollout strategy originally proposed by \citet{hou2025treerl}. In the first phase, the model generates $N$ independent base rollouts in parallel. Subsequently, in the second phase, we apply a forking point selection algorithm to identify $K$ forking points within each base trajectory. From each selected forking point, we generate $B$ additional branch rollouts, thereby expanding the exploration scope. Detailed algorithmic procedures are provided in \hyperref[app:tree_algorithm]{Appendix~\ref{app:tree_algorithm}}.

\subsubsection{Empirical Analysis}
\label{sec:tree_or_parallel_analysis}
To analyze cost-efficiency during inference, we compare standard parallel sampling against two tree-structured rollout variants by analyzing PassRate---the probability of obtaining at least one correct solution---as a function of token consumption during generation. These variants include \texttt{fixed-seg}, which selects the $K$ forking points at equidistant intervals, and \texttt{random}, which selects them uniformly at random (detailed in \hyperref[app:experimental_details_tree_parallel]{Appendix~\ref{app:experimental_details_tree_parallel}}).

As shown in \autoref{fig:tree_parallel}, both tree-structured variants achieve higher token efficiency than parallel sampling. Specifically, their performance curves demonstrate a steeper growth rate, yielding a higher PassRate for the same number of generated tokens. This enhanced efficiency stems from prefix sharing in tree structures, which avoids redundant regeneration of identical early segments and enables the exploration of alternative continuations within the same token budget \citep{tran2025exploiting, hou2025treerl}.

Furthermore, the choice of forking strategy also affects performance within tree-structured methods. While both \texttt{fixed-seg} and \texttt{random} outperform parallel sampling, \texttt{fixed-seg} consistently achieves a higher PassRate under comparable token budgets. This demonstrates that the effectiveness of tree-based exploration depends not only on its structure but also on the placement of branching points.

\subsection{Forking Point Selection in Tree Search}
\label{sec:optimal_forking}
\hyperref[sec:tree_vs_parallel]{Section~\ref{sec:tree_vs_parallel}} shows that the performance of tree-structured rollouts varies substantially across different forking strategies. In this section, we investigate which forking strategies most effectively identify correct solutions and promote semantically diverse reasoning paths.

\subsubsection{The Localization Phenomenon and Sentence-level Forking}
\label{sec:forking_method}
Previous studies primarily select top-$k$ high-entropy tokens as forking points \citep{hou2025treerl, zheng2025first, cao2026treeadv}, as they capture highly uncertain words that act as pivotal branching points \citep{wang2025beyond, cheng2025reasoning}. However, we identify a critical limitation in this approach: the \textit{localization} phenomenon (\autoref{fig:localization}). High-entropy tokens tend to densely cluster within narrow, highly uncertain segments of the reasoning trajectory. Consequently, the search budget is monopolized by repeated resampling within a localized cluster, which limits the structural reach of the search tree.

To mitigate localization, we propose a sentence-level forking strategy (\texttt{sent-entropy}), which expands the granularity of entropy estimation. In this approach, sentence entropy is calculated as the average entropy of its tokens, and the starting points of the top-$k$ high-entropy sentences are selected as forking points. This strategy naturally evades localization while still targeting the most uncertain regions. We compare this approach with the conventional token-level method (\texttt{tok-entropy}) and three baselines: \texttt{random}, fixed-segment selection (\texttt{fixed-seg}), and Attention-based Tree Branching (\texttt{ATB}) \citep{liu2025attention}, which branches at steps receiving the highest attention weights. Detailed algorithmic procedures are provided in \hyperref[app:forking_point_algorithm]{Appendix~\ref{app:forking_point_algorithm}}.

\subsubsection{Metrics for Measuring Diversity}
\label{sec:forking_metric}
While PassRate effectively captures task success, comprehensively evaluating these forking strategies also requires measuring whether the generated paths are semantically diverse. To directly quantify exploration diversity, we introduce an embedding-based metric, Sibling Diversity (SibDiv). We first partition the reasoning tree into contiguous text blocks bounded by forking points or the end of the trajectory. SibDiv then computes the average pairwise cosine distance (i.e., $1 - \text{cosine similarity}$) between the embeddings of sibling blocks originating from the same forking point. By aggregating these values, this metric effectively evaluates how well the forking points promote semantically diverse reasoning branches. Formal definitions are provided in \hyperref[app:diversity_metrics]{Appendix~\ref{app:diversity_metrics}}.

\begin{table}[t]
\centering
\renewcommand{\arraystretch}{1.0}
\begin{tabular}{lcc}
\toprule
\textbf{Method} & \textbf{PassRate} & \textbf{SibDiv} \\
\midrule
\texttt{random}        & 10.0 & 0.0796 \\
\texttt{fixed-seg}     & 13.3 & 0.0853 \\
\texttt{ATB}           & \underline{14.7} & 0.0874 \\
\texttt{tok-entropy}   & 12.0 & \textbf{0.1065} \\
\texttt{sent-entropy}  & \textbf{17.3} & \underline{0.1049} \\
\bottomrule
\end{tabular}
\caption{Comparison of \textbf{forking point selection strategies} when applied during inference. Best results are in \textbf{bold}, and second-best results are \underline{underlined}.}
\label{tab:forking_point_results}
\end{table}

\subsubsection{Empirical Analysis}
\label{sec:forking_analysis}
We evaluated each forking strategy by generating reasoning trees under a fixed inference budget, measuring PassRate and SibDiv (see \hyperref[app:experimental_details_forking_points]{Appendix \ref{app:experimental_details_forking_points}} for detailed experimental setups).

The results are summarized in \autoref{tab:forking_point_results}. First, \texttt{tok-entropy} achieves the highest SibDiv because performing additional sampling at the model's most uncertain points naturally generates diverse immediate sibling blocks. However, due to the \textit{localization} phenomenon, this high SibDiv is strictly confined to a narrow segment of the reasoning path. The search fails to expand into meaningful structural differences across the entire reasoning tree, resulting in a low PassRate. By simply expanding the granularity of the entropy estimation, \texttt{sent-entropy} shows the highest PassRate while incurring minimal loss in SibDiv. In contrast, baselines such as \texttt{fixed-seg} and \texttt{ATB} inherently avoid localization, yielding higher PassRates than \texttt{tok-entropy}, but exhibit lower SibDiv since they do not utilize uncertainty signals. Ultimately, \texttt{sent-entropy} emerges as the most effective strategy by successfully translating high semantic diversity into a high PassRate.

\section{Methodology}
\label{sec:methodology}
Building upon the empirical insights from \hyperref[sec:3_analysis]{Section~\ref{sec:3_analysis}}, we propose \textbf{DATPO} (\textbf{D}ifficulty-\textbf{A}daptive Sentence-entropy-guided \textbf{T}ree-structured \textbf{P}olicy \textbf{O}ptimization) to expand a model's intrinsic reasoning coverage (pass@$k$). An overview of the proposed framework is illustrated in \autoref{fig:overview}.

\subsection{Difficulty-Adaptive Tree Search at Train-time}
\label{sec:method_adaptive_tree}

Building on the tree algorithm proposed in \hyperref[sec:tree_rollout]{Section~\ref{sec:tree_rollout}}, we introduce a difficulty-adaptive tree search to dynamically allocate the search budget. The procedure operates in two phases. First, we generate $N$ independent base rollouts and estimate the empirical difficulty of the prompt through their average verifiable reward $V(\mathrm{root}) \in [0, 1]$.

Second, we scale the tree expansion proportionally to this difficulty, where a lower $V(\mathrm{root})$ indicates a more challenging problem, thus allocating more search budget. Given maximum budgets for forking points $K_{\max}$ and branch rollouts $B_{\max}$, the adaptive parameters are computed as:
\begin{align}
    \hat{K} &= \lceil K_{\max} (1 - V(\mathrm{root})) \rceil \\
    \hat{B} &= \lceil B_{\max} (1 - V(\mathrm{root})) \rceil
\end{align}

Using the \texttt{sent-entropy} (\hyperref[sec:forking_method]{Section~\ref{sec:forking_method}}), we identify $\hat{K}$ forking points within each base trajectory and generate $\hat{B}$ branches from each point.

This mechanism entirely bypasses expansion only when $V(\mathrm{root}) = 1$. For harder problems, it expands the search space up to $N(1 + \hat{K}\hat{B})$ leaves. By scaling expansion based on difficulty, this strategy concentrates exploration on unsolved problems to maximize reasoning coverage.

\subsection{Block-level Diversity-augmented Advantage Estimation}
\label{sec:method_advantage}

To optimize the policy, we partition trajectories into contiguous \textit{blocks} bounded by forking points or the end of the trajectory. Advantages are computed and assigned at this block level.

Since our tree search generates multiple branches from intermediate states, we reliably estimate state values using Monte Carlo (MC) returns \citep{kazemnejad2024vineppo}. For a state $s$ at any forking point, $\hat{V}_{MC}(s)$ is the average verifiable reward of all its descending terminal blocks. For a terminal state $s_T$ without further rollouts, $\hat{V}_{MC}(s_T) = 0$.

For a block $b$ spanning from $s_{\text{start}}^{(b)}$ to $s_{\text{end}}^{(b)}$, the base advantage is formulated as:
\begin{equation}
    \hat{A}_{\text{base}}(b) = r(b) + \hat{V}_{MC}(s_{\text{end}}^{(b)}) - \hat{V}_{MC}(s_{\text{start}}^{(b)})
\end{equation}
where the verifiable reward $r(b) \in \{0, 1\}$ is assigned exclusively to terminal blocks; otherwise, it is $0$. All tokens within $b$ share this identical advantage. This block-level advantage assigns precise credit to intermediate steps, effectively facilitating implicit process supervision.

\begin{figure}[t]
    \centering
    \includegraphics[width=\linewidth]{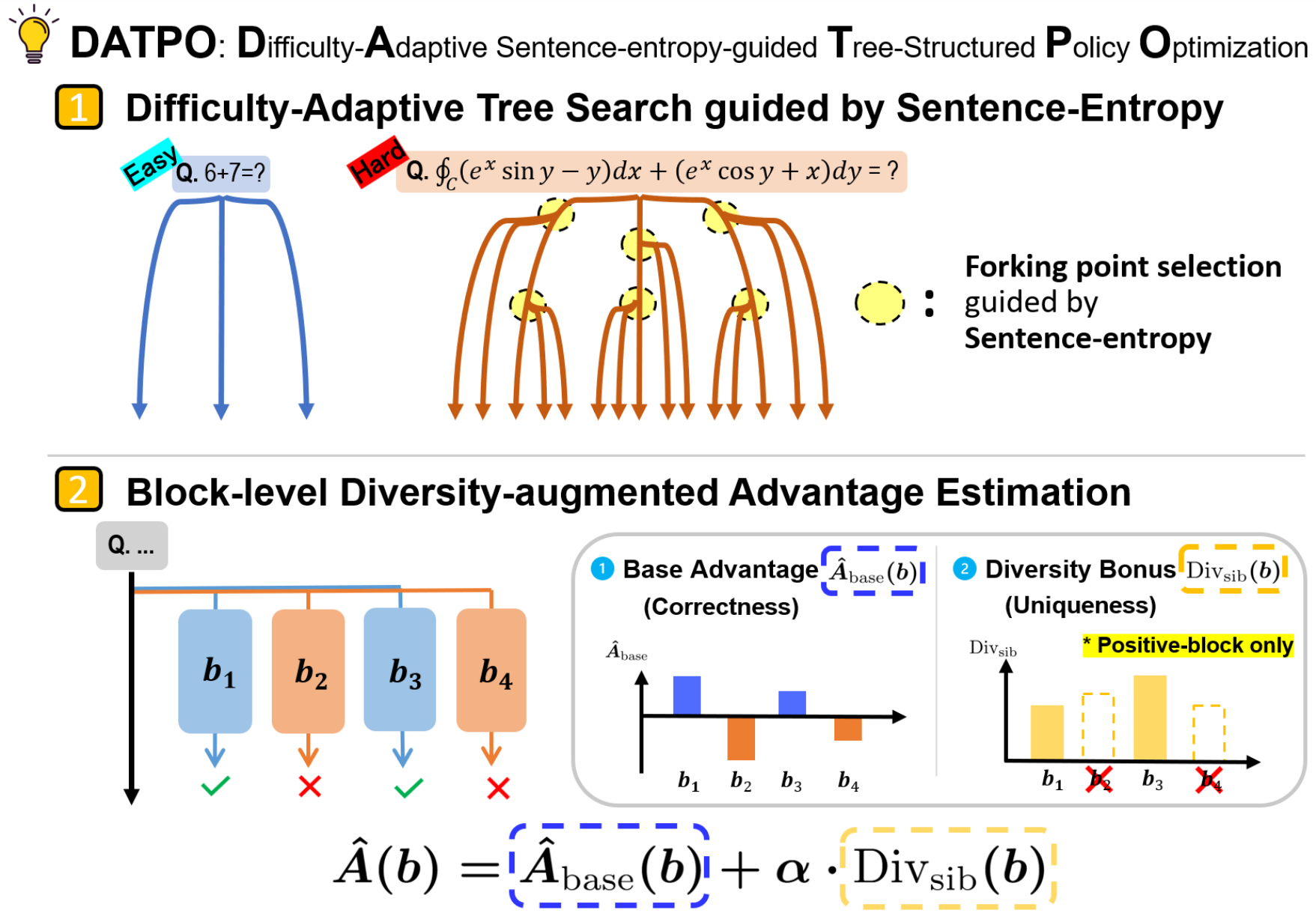}
    \caption{Overview of the proposed DATPO framework.}
    \label{fig:overview}
\end{figure}

To maximize the coverage-expanding benefits, we augment the base advantage with a \textit{sibling-diversity} term $\text{Div}_{\text{sib}}(b)$, inspired by the SibDiv metric. Specifically, sibling blocks refer to the child blocks generated from a shared forking point. $\text{Div}_{\text{sib}}(b)$ is defined as the average cosine distance between the embedding of block $b$ and its sibling blocks, effectively encouraging the model to discover distinct reasoning paths. Crucially, we apply this bonus exclusively to blocks with positive base advantages to avoid incentivizing the exploration of incorrect paths. The augmented advantage is:
\begin{equation}
    \hat{A}(b) = \hat{A}_{\text{base}}(b) + \mathbb{I}(\hat{A}_{\text{base}}(b) > 0) \cdot \alpha \cdot \text{Div}_{\text{sib}}(b)
\end{equation}
where $\mathbb{I}(\cdot)$ is the indicator function, and the coefficient $\alpha$ is linearly annealed during training to gradually decay the exploration incentive, allowing the model to refine its learned reasoning paths.

Finally, we optimize the policy directly across the generated tree topology. Given a set of contiguous blocks $\mathcal{B}$ generated for a prompt $q$, the DATPO objective is formulated as:
\begin{equation}
\begin{split}
\mathcal{J}_{\mathrm{DATPO}}(\theta) &= \mathbb{E}_{q \sim \mathcal{Q}, \mathcal{B} \sim \pi_{\theta_{\mathrm{old}}}} \Bigg[ \frac{1}{\sum_{b \in \mathcal{B}} |b|} \sum_{b \in \mathcal{B}} \sum_{t=1}^{|b|} \\
&\quad \min \Big( \rho_{b,t}(\theta) \hat{A}(b), \mathrm{clip}\big(\rho_{b,t}(\theta), \\
&\qquad\quad 1-\epsilon, 1+\epsilon\big) \hat{A}(b) \Big) \Bigg],
\end{split}
\end{equation}
where $|b|$ is the sequence length of block $b$, and $\hat{A}(b)$ is the block-level augmented advantage. Comprehensive algorithmic details and training procedures are provided in \autoref{app:datpo_details}.

\begin{table*}[t!]
\centering
\small
\setlength{\tabcolsep}{4pt}
\renewcommand{\arraystretch}{1.15}

\resizebox{\textwidth}{!}{
\begin{tabular}{lcccccccccccc}
\toprule

\multirow{2}{*}{\textbf{Method}}
& \multicolumn{2}{c}{\textbf{MATH500}}
& \multicolumn{2}{c}{\textbf{AIME26}}
& \multicolumn{2}{c}{\textbf{AIME25}}
& \multicolumn{2}{c}{\textbf{AIME24}}
& \multicolumn{2}{c}{\textbf{AMC23}}
& \multicolumn{2}{c}{\textbf{Average}} \\

\cmidrule(lr){2-3} \cmidrule(lr){4-5} \cmidrule(lr){6-7} \cmidrule(lr){8-9} \cmidrule(lr){10-11} \cmidrule(lr){12-13}

& avg@8 & pass@8
& avg@64 & pass@64
& avg@64 & pass@64
& avg@64 & pass@64
& avg@64 & pass@64
& avg@$k$ & pass@$k$ \\

\midrule

\rowcolor{gray!15}
\multicolumn{13}{c}{\textbf{Qwen2.5-3B-Base}} \\
Base          & 26.3 & 64.7 & 0.4 & 8.9 & 0.2 & 12.2 & 1.0 & 22.2 & 10.0 & 78.3 & 7.6 & 37.3 \\
GRPO          & 61.8 & 79.8 & 1.7 & 21.1 & 0.7 & 24.4 & 3.6 & 30.0 & 35.9 & \underline{85.8} & 20.7 & 48.2 \\
Dr.GRPO       & 61.3 & 78.7 & \textbf{3.0} & 23.3 & 1.4 & \underline{27.8} & \underline{5.1} & 27.8 & 37.9 & 83.3 & \underline{21.7} & 48.2 \\
TreeRL        & 61.7 & 80.7 & \underline{2.4} & 22.2 & \textbf{1.7} & 20.0 & 4.2 & 24.4 & \underline{38.7} & 83.3 & \underline{21.7} & 46.1 \\
AttnRL        & \underline{62.0} & \textbf{82.1} & 1.9 & \underline{32.2} & 1.5 & 25.6 & \textbf{5.2} & \underline{34.4} & 35.7 & \textbf{90.8} & 21.3 & \underline{53.0} \\
DATPO (Ours)  & \textbf{63.5} & \underline{81.7} & \underline{2.4} & \textbf{33.3} & \underline{1.6} & \textbf{33.3} & 4.8 & \textbf{35.6} & \textbf{39.8} & \textbf{90.8} & \textbf{22.4} & \textbf{54.9} \\

\midrule
\rowcolor{gray!15}
\multicolumn{13}{c}{\textbf{Qwen3-4B-Base}} \\
Base          & 39.9 & 79.4 & 1.7 & 21.1 & 1.1 & 32.2 & 2.8 & 38.9 & 17.2 & 84.2 & 12.5 & 51.2 \\
GRPO          & \underline{75.5} & 87.4 & 7.1 & \underline{32.2} & \underline{10.3} & \underline{38.9} & 9.6 & 41.1 & 47.9 & \textbf{91.7} & 30.1 & \underline{58.3} \\
Dr.GRPO       & 75.2 & 86.2 & \underline{7.8} & 25.6 & 7.7 & 35.6 & 9.7 & 40.0 & 46.6 & \underline{90.8} & 29.4 & 55.6 \\
TreeRL        & 75.3 & 86.2 & \textbf{7.9} & 28.9 & 9.4 & 36.7 & 9.4 & 41.1 & \underline{49.4} & \underline{90.8} & 30.3 & 56.7 \\
AttnRL        & 74.7 & \underline{87.5} & 7.1 & 28.9 & 8.2 & 36.7 & \underline{11.0} & \underline{42.2} & \textbf{52.5} & \textbf{91.7} & \underline{30.7} & 57.4 \\
DATPO (Ours)  & \textbf{76.4} & \textbf{87.9} & 6.5 & \textbf{33.3} & \textbf{10.9} & \textbf{42.2} & \textbf{14.0} & \textbf{46.7} & 48.4 & \textbf{91.7} & \textbf{31.3} & \textbf{60.4} \\

\bottomrule
\end{tabular}
}
\caption{Evaluation results on mathematical reasoning benchmarks. We report avg@$k$ and pass@$k$ for each dataset.}
\label{tab:main_results}
\end{table*}

\section{Experiments}
\label{sec:experiments}
\subsection{Experimental Setup}
\label{sec:exp_setup}
\paragraph{Models and Datasets}
\label{par:models_datasets}
We use Qwen2.5-3B-Base \citep{qwen2.5} and Qwen3-4B-Base \citep{qwen3} as base models. For training, we use the MATH dataset \citep{hendrycks2021measuring}.

\paragraph{Evaluation}
\label{par:evaluation_metrics}
We evaluate the trained models on mathematical reasoning benchmarks, specifically MATH500 \citep{hendrycks2021measuring}, AIME26, AIME25, AIME24, and AMC23. Using a temperature of 1.0, we report avg@$k$ as the primary evaluation metric across all benchmarks. Additionally, we report pass@$k$ to assess the reasoning coverage of the trained models. For robustness, the pass@$k$ results are computed by averaging over three independent evaluation runs.

\begin{figure}[t]
    \centering
    \includegraphics[width=0.85\linewidth]{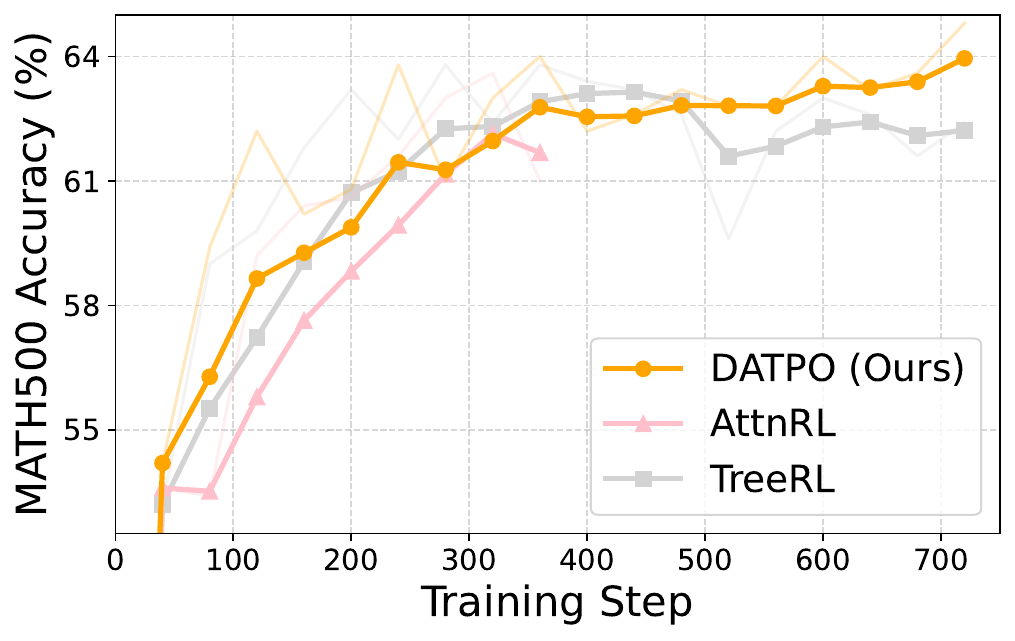}
    \caption{Learning curves of MATH500 accuracy over training steps for tree-based methods. (Smoothed)}
    \label{fig:learning_curve}
\end{figure}

\paragraph{Baselines}
\label{par:baselines}
We compare DATPO against the \textbf{Base} model and several advanced RLVR baselines. These include \textbf{GRPO} \citep{shao2024deepseekmath} (enhanced with clip-higher and token-level loss \citep{yu2025dapo}) and its variant, \textbf{Dr.GRPO} \citep{liu2025understanding}. Furthermore, we evaluate two tree-based methods: \textbf{TreeRL} \citep{hou2025treerl}, which utilizes top-$k$ high-entropy forking and combined local-global advantages, and \textbf{AttnRL} \citep{liu2025attention}, which leverages attention-based forking, adaptive sampling, and a one-step off-policy for enhanced efficiency. To ensure a fair comparison, we maintain a comparable total number of generated tokens per problem across all methods while adopting the tree-expansion hyperparameters reported in the original TreeRL and AttnRL papers. Comprehensive details for experiments are provided in \autoref{app:experiments}.

\begin{figure}[t]
    \centering
    \includegraphics[width=\linewidth]{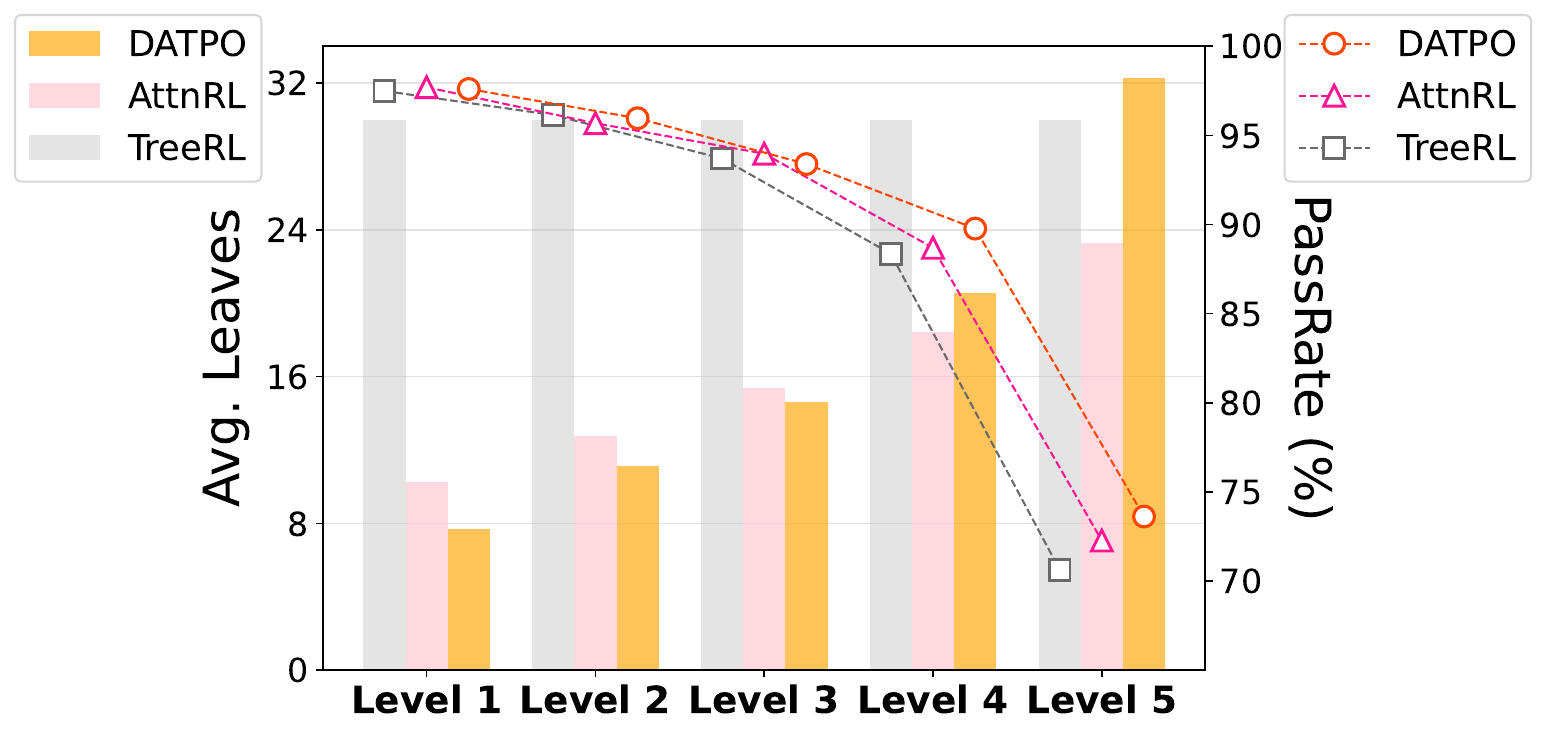}
    \caption{Comparison of average leaf count generated at train-time and PassRate across difficulty levels for tree-based methods.}
    \label{fig:leaf_count}
\end{figure}

\subsection{Main Results}
\label{sec:main_results}
\paragraph{Mathematical Reasoning Performance}
\label{par:math_reasoning_performance}
As reported in \autoref{tab:main_results}, DATPO achieves the best aggregate avg@k and pass@k across the evaluated benchmarks. Notably, while the improvements in single-sample accuracy (avg@$k$) are marginal compared to the strongest baseline, AttnRL (e.g., +1.1 and +0.6 on Qwen2.5-3B-Base and Qwen3-4B-Base, respectively), DATPO achieves substantial gains in pass@$k$, outperforming AttnRL by +1.9 and +3.0, respectively. This indicates that DATPO's difficulty-adaptive rollout and sibling-diversity term successfully expand the model's intrinsic reasoning coverage.

\paragraph{Analysis of Training Dynamics}
\label{par:training_dynamics_analysis}
We first examine the MATH500 accuracy over training steps (\autoref{fig:learning_curve}). While all three tree-based methods exhibit similar accuracy gains during the initial training steps, TreeRL and AttnRL fail to sustain this upward trend, eventually plateauing or even suffering performance degradation. In contrast, DATPO shows a consistent upward trend throughout the training. This stability can be attributed to the annealed sibling-diversity term, which injects semantic diversity early in training to prevent premature convergence.

Furthermore, we analyze the search behavior during training by comparing the average leaf count and PassRate across the difficulty levels of the MATH dataset (\autoref{fig:leaf_count}). Note that these difficulty labels are used strictly for this analysis and were not provided to the model during training. Unlike TreeRL, which maintains a fixed leaf count across all difficulties, DATPO and AttnRL employ difficulty-adaptive rollouts. However, DATPO allocates fewer rollouts to easy problems and significantly more rollouts to hard problems compared to AttnRL. The latter's adaptive strategy relies heavily on an attention-based filtering mechanism that simply discards problems with below-average attention scores to maximize computational efficiency. Because this approach acts as a rigid cut-off, it fails to concentrate computational resources on the most challenging problems. Consequently, DATPO's highly adaptive resource allocation enables it to achieve a higher PassRate on harder problems (Levels 4 and 5) compared to other tree-based methods.

\begin{figure}[t]
    \centering
    \includegraphics[width=0.85\linewidth]{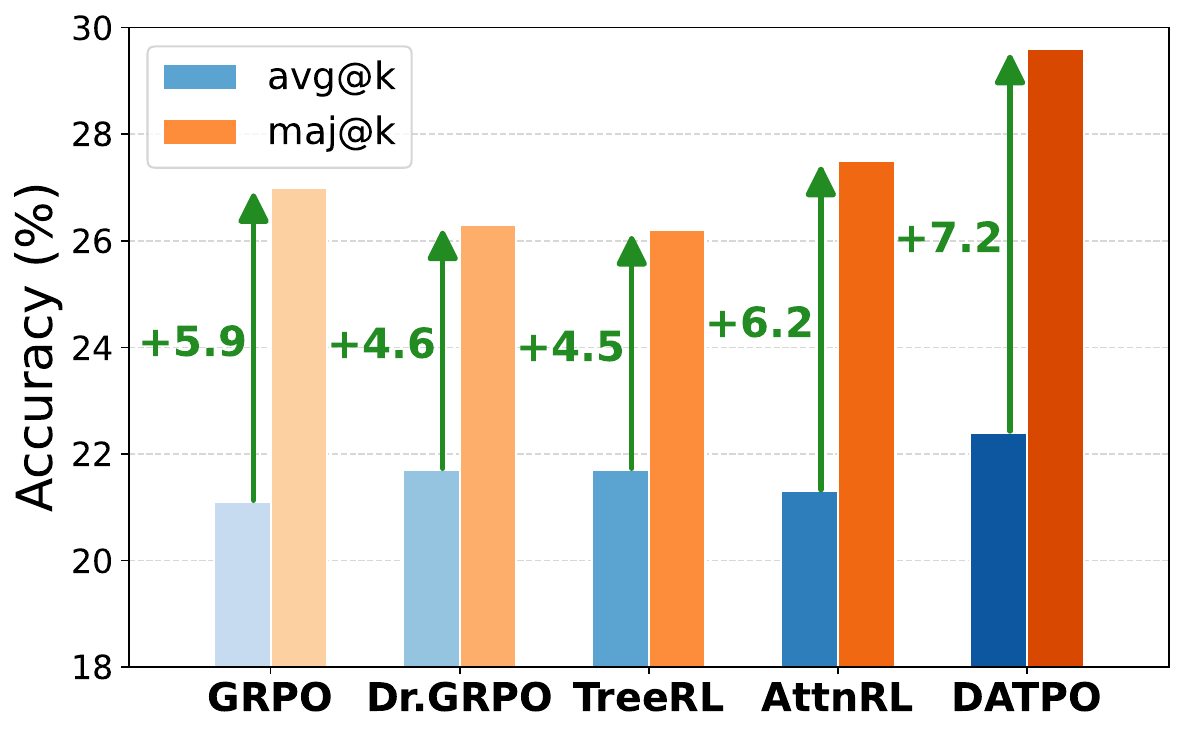}
    \caption{Test-time scaling performance using majority voting (maj@$k$). The green arrows indicate the performance improvement of maj@$k$ over avg@$k$.}
    \label{fig:test_time_scaling}
\end{figure}

\subsection{Test-time Scaling Performance}
\label{sec:test_time_scaling_performance}
In this section, we investigate whether the substantial gains in pass@$k$ directly translate into improved test-time scaling performance. To this end, we apply majority voting (maj@$k$) \citep{wang2022self}---one of the simplest and most widely adopted test-time scaling strategies---to compare the methods applied to Qwen2.5-3B-Base. We evaluate the effectiveness of these methods across five mathematical reasoning benchmarks, applying $k=8$ for MATH500 and $k=64$ for the remaining datasets. For robustness, all maj@$k$ results are computed as the average over three independent evaluation runs.

As illustrated by the average performance across five benchmarks in \autoref{fig:test_time_scaling}, DATPO, which achieved the highest pass@$k$ during training, also shows the largest maj@$k$ gain, improving by +7.2 over its avg@$k$. This demonstrates that expanding a model's reasoning coverage during training directly enhances test-time scaling performance.

\subsection{Ablation Studies}
\label{sec:ablation}

\paragraph{Effect of Forking Strategies}
\label{par:ablation_forking}
Building upon the analysis during inference presented in \hyperref[sec:optimal_forking]{Section~\ref{sec:optimal_forking}}, we further investigate the impact of different forking strategies during training. We compare DATPO's default \texttt{sent-entropy} strategy against four alternatives: \texttt{random}, \texttt{fixed-seg}, \texttt{ATB}, and \texttt{tok-entropy}, evaluating their performance on the MATH500 benchmark. As reported in \autoref{tab:ablation_forking_strategies}, the \texttt{sent-entropy} yields the highest performance. This confirms our hypothesis: the \texttt{sent-entropy} effectively guides exploration using the model's uncertainty, while avoiding \textit{localization} that limits \texttt{tok-entropy}.

\begin{table}[t]
\centering
\footnotesize
\setlength{\tabcolsep}{8pt}
\renewcommand{\arraystretch}{1.2}
\begin{tabular}{lcc}
\toprule
\multirow{2}{*}{\textbf{Method}} & \multicolumn{2}{c}{\textbf{MATH500}} \\
\cmidrule(lr){2-3}
 & \textbf{avg@8} & \textbf{pass@8} \\
\midrule
\texttt{random}       & 61.5 & \underline{81.6} \\
\texttt{fixed-seg}    & 61.1 & 80.6 \\
\texttt{ATB}          & \underline{62.1} & 80.7 \\
\texttt{tok-entropy}  & 61.3 & 80.3 \\
\texttt{sent-entropy} & \textbf{63.5}  & \textbf{81.7} \\
\bottomrule
\end{tabular}
\caption{Ablation study on different forking strategies.}
\label{tab:ablation_forking_strategies}
\end{table}

\paragraph{Effect of Sibling-Diversity}
\label{par:ablation_alpha}
To evaluate the effect of the sibling-diversity term, we conduct an ablation study by varying the coefficient $\alpha$, and evaluate the performance on the MATH500 benchmark, as summarized in \autoref{tab:ablation_alpha}.

The results show that the $0.2 \rightarrow 0$ schedule is the most effective, surpassing both the baseline without diversity ($0 \rightarrow 0$) and a higher initial coefficient ($0.4 \rightarrow 0$). Furthermore, applying the diversity bonus to all blocks ($0.2 \rightarrow 0$, \textit{all}) degrades performance compared to the baseline without diversity, highlighting the need to reward diversity exclusively on valid paths. Annealing is also critical: a constant $\alpha$ ($0.2 \rightarrow 0.2$) severely degrades avg@8 despite slight pass@8 gains. Conversely, annealing $\alpha$ to a negative value ($0.2 \rightarrow -0.2$) improves avg@8 but noticeably harms pass@8, confirming that penalizing diversity in the later stages of training restricts the model's reasoning coverage.

\section{Related Work}
\paragraph{Reinforcement Learning for LLM Reasoning}
RLVR has become the standard for post-training of reasoning LLMs. A representative method is Group Relative Policy Optimization (GRPO) \citep{shao2024deepseekmath}, which replaces the costly critic model of PPO \citep{schulman2017proximal} with efficient group-based advantage estimation. Recent studies build on this framework to enhance stability and resolve optimization issues: DAPO \citep{yu2025dapo} prevents mode collapse through decoupled clipping and dynamic sampling, while Dr.GRPO \citep{liu2025understanding} corrects inherent structural biases in the advantage computation.

\begin{table}[t]
\centering
\footnotesize
\setlength{\tabcolsep}{8pt}
\renewcommand{\arraystretch}{1.2}
\begin{tabular}{lcc}
\toprule
\multirow{2}{*}{\textbf{$\alpha$}} & \multicolumn{2}{c}{\textbf{MATH500}} \\
\cmidrule(lr){2-3}
 & \textbf{avg@8} & \textbf{pass@8} \\
\midrule
0 $\rightarrow$ 0 & 62.1 & 81.4 \\
\textbf{0.2 $\rightarrow$ 0} & \textbf{63.5} & \underline{81.7} \\
0.4 $\rightarrow$ 0 & 61.9 & 81.3 \\
0.2 $\rightarrow$ 0, \textit{all} & 61.6 & 81.3 \\
0.2 $\rightarrow$ 0.2 & 61.5 & \textbf{82.0} \\
0.2 $\rightarrow$ -0.2 & \underline{63.4} & 81.0 \\
\bottomrule
\end{tabular}
\caption{Ablation study on the diversity coefficient $\alpha$. The right arrow ($\rightarrow$) indicates the linear annealing schedule during training. The \textit{all} applies the diversity term to all blocks, rather than exclusively to those with positive base advantages.}
\label{tab:ablation_alpha}
\end{table}

\paragraph{Exploration Strategies in RLVR}
RLVR often fails to expand a model's intrinsic reasoning capacity (pass@$k$) beyond its base capabilities \citep{yue2025does, dang2025assessing, wu2025invisible}. To overcome this exploration bottleneck, recent studies introduce explicit mechanisms to diversify search. PKPO \citep{walder2025pass} directly optimizes pass@$k$ to solve harder instances, while R1-zero-Div \citep{yao2025diversity} and FOR \citep{yuflow} use diversity-aware objectives and divergent reasoning flows to prevent trajectory collapse.

\paragraph{Tree-based Search in Reinforcement Learning}
Proven in RL milestones like AlphaGo \citep{silver2016mastering}, tree-based search is now adapted to LLMs to enable systematic exploration. TreeRL \citep{hou2025treerl} pioneers this with an on-policy framework using entropy-guided branching. To address computational bottlenecks, AttnRL \citep{liu2025attention} improves efficiency by leveraging difficulty-aware adaptive sampling within a one-step off-policy pipeline. While these methods use trees primarily for credit assignment or computational efficiency, we leverage them to expand the model's reasoning coverage.

\section{Conclusion}
In this work, we investigate the structural design of train-time rollouts in RLVR to expand the intrinsic reasoning coverage. Our analysis reveals that expanding reasoning coverage can benefit from moving beyond uniform parallel sampling toward difficulty-adaptive resource allocation and the structural advantages of sentence-entropy-guided tree search. Integrating these principles, we propose DATPO, which explicitly encourages semantic exploration through a sibling-diversity term. Experiments on mathematical benchmarks validate our approach, showing significant improvements in pass@$k$.

\section*{Limitations}
\label{sec:limitation}
While DATPO significantly expands reasoning coverage, calculating the sibling-diversity term requires additional forward passes through an external embedding model, which introduces computational overhead.

Additionally, while our block-level formulation effectively assigns process-level credit across the tree topology, it relies on small sample sizes. Deriving both the empirical difficulty and Monte Carlo state values from a limited number of rollouts (e.g., $N=4, B=4$) can inject noise into the estimations, occasionally yielding high variance during policy updates.

Finally, resource limitations restricted our empirical validation to the 3B-4B parameter regime, leaving DATPO's scalability to larger models ($\ge 7$B) unverified. Furthermore, as our evaluation primarily focuses on mathematical reasoning, the framework's direct effectiveness in other rigorous domains, such as complex logical reasoning or code generation, requires further investigation.

\section*{Acknowledgments}
This work was supported by the Institute of Information \& Communications Technology Planning \& Evaluation (IITP) grants funded by the Korea government (MSIT) (No. RS-2019-II191906, Artificial Intelligence Graduate School Program (POSTECH); IITP-2026-RS-2026-25616370, AI Star Fellowship Support Program) and by the National Research Foundation of Korea (NRF) grant funded by the Korea government (MSIT) (No. RS-2024-00335873).

\bibliography{custom}

\clearpage

\setcounter{theorem}{0}
\renewcommand{\thetheorem}{\thesection.\arabic{theorem}}
\appendix

\section{Theoretical Analysis of Difficulty-Adaptive Rollout}
\label{app:adaptive_theory}

We provide a theoretical analysis of the empirical findings in \hyperref[sec:adaptive_rollout]{Section~\ref{sec:adaptive_rollout}}. In particular, we explain why increasing the rollout budget $G$ consistently improves the expected avg@$k$, while the behavior of pass@$k$ can vary across difficulty subsets.

\paragraph{Notation.}
For a prompt $x$ and a policy $\pi_\theta$, let
\[
p_x(\theta) := \Pr_{y \sim \pi_\theta(\cdot \mid x)}[r(y)=1]
\]
denote the single-sample correctness probability under a binary verifiable reward
$r(y)\in\{0,1\}$.

In empirical evaluations, metrics are computed over a finite set of $k$ independent samples $\mathcal{Y} = (y^{(1)}, \dots, y^{(k)}) \stackrel{\mathrm{i.i.d.}}{\sim} \pi_\theta(\cdot \mid x)$. We define the empirical estimators as:
\[
\begin{aligned}
\widehat{\text{avg@k}}_x(\theta) &:= \frac{1}{k}\sum_{j=1}^k r\!\left(y^{(j)}\right), \\
\widehat{\text{pass@k}}_x(\theta) &:= \mathbb{I}\!\left(\sum_{j=1}^k r\!\left(y^{(j)}\right) \ge 1\right).
\end{aligned}
\]

However, to rigorously analyze the optimization dynamics, we focus on their \textit{expected values} with respect to the policy. Throughout this analysis, we omit the hat notation and define $\text{avg@k}_x(\theta)$ and $\text{pass@k}_x(\theta)$ strictly as these theoretical expectations:
\[
\begin{aligned}
\text{avg@k}_x(\theta)
&:=
\mathbb{E}_{\mathcal{Y} \sim \pi_\theta}
\left[
\widehat{\text{avg@k}}_x(\theta)
\right] \\
&= p_x(\theta)
\end{aligned}
\]
and
\[
\begin{aligned}
\text{pass@k}_x(\theta)
&:=
\mathbb{E}_{\mathcal{Y} \sim \pi_\theta}
\left[
\widehat{\text{pass@k}}_x(\theta)
\right] \\
&= 1-\bigl(1-p_x(\theta)\bigr)^k.
\end{aligned}
\]
By formulating these metrics as expectations, the theoretical avg@$k$ rigorously simplifies to the single-sample success probability $p_x(\theta)$.

For the analysis below, given the current policy parameters $\theta$, let
\[
\theta_G^{+} := \theta + \eta \hat g_G(x)
\]
denote one GRPO update with rollout group size $G$, learning rate $\eta$, and stochastic gradient estimator $\hat g_G(x)$. We consider only the reward-advantage component of the GRPO update, excluding auxiliary regularization terms.

\begin{theorem}[\textbf{One-step expected avg@$k$ improvement with rollout budget}]
\label{thm:avgk_monotone}
Fix a prompt $x$ and a policy $\pi_\theta$. Let $p := p_x(\theta) \in (0,1)$, and let $S_G \sim \mathrm{Binomial}(G, p)$ denote the number of correct responses in a rollout group size $G \ge 2$.

Assume $p_x(\cdot)$ is twice continuously differentiable with a bounded Hessian, and the gradient estimator satisfies $\mathbb{E}[\|\hat g_G(x)\|^2] < \infty$. Further assume:
\begin{enumerate}
    \renewcommand{\labelenumi}{\textbf{A\arabic{enumi}.}}
    \item \label{ass:zero_update} The update vanishes if all group rewards are identical: $\hat g_G(x) = 0$ almost surely if $S_G \in \{0, G\}$.
    \item \label{ass:constant_cx} The expected update direction on mixed-reward batches is a positive constant $c_x > 0$ independent of $G$:
    \[ 
    \begin{aligned}
        &\mathbb{E}\!\left[ \langle \nabla_\theta p_x(\theta), \hat g_G(x) \rangle \mid 1 \le S_G \le G-1 \right] \\
        &\quad = c_x . 
    \end{aligned}
    \]
\end{enumerate}

Then, for sufficiently small $\eta$, the expected correctness probability satisfies:
\[
\mathbb{E}[p_x(\theta_G^{+})] = p_x(\theta) + \eta c_x \Psi_G(p) + O_G(\eta^2),
\]
where $\Psi_G(p) := 1 - p^G - (1-p)^G$. 

Consequently, for any finite set of rollout budgets $\mathcal{G}$, there exists $\eta_0 > 0$ such that for all $\eta \in (0, \eta_0)$ and any $G_1, G_2 \in \mathcal{G}$ with $G_2 > G_1$:
\[
\mathbb{E}[p_x(\theta_{G_2}^{+})] > \mathbb{E}[p_x(\theta_{G_1}^{+})].
\]
Furthermore, the first-order gain coefficient $\Psi_G(p)$ is strictly increasing in $G$ with diminishing returns.
\end{theorem}

\begin{proof}
By definition, theoretical expected $\text{avg@k}$ simplifies to the single-sample correctness probability $p_x(\theta)$. Let $E_G := \{1 \le S_G \le G-1\}$ be the event of a mixed-reward batch, which occurs with probability $\Pr(E_G) = 1 - (1-p)^G - p^G = \Psi_G(p)$.

By \hyperref[ass:zero_update]{A\ref{ass:zero_update}}, the update is zero outside $E_G$. Applying \hyperref[ass:constant_cx]{A\ref{ass:constant_cx}}, the expected directional derivative is:
\[
\begin{aligned}
&\mathbb{E}[\langle \nabla_\theta p_x(\theta), \hat g_G(x) \rangle] \\
&\quad= \Pr(E_G) \mathbb{E}[\langle \nabla_\theta p_x(\theta), \hat g_G(x) \rangle \mid E_G] \\
&\quad= c_x \Psi_G(p).
\end{aligned}
\]
Since $p_x(\cdot)$ has a bounded Hessian on the update region, Taylor's theorem
gives
\[
p_x(\theta_G^{+})
=
p_x(\theta)
+
\eta
\left\langle \nabla_\theta p_x(\theta), \hat g_G(x) \right\rangle
+
R_G(\eta),
\]
where, for some Hessian bound $L<\infty$, $\left|R_G(\eta)\right| \le \frac{L}{2}\eta^2\|\hat g_G(x)\|^2.$
Taking expectations and using
$\mathbb{E}[\|\hat g_G(x)\|^2]<\infty$ gives
$\mathbb{E}[R_G(\eta)]=O_G(\eta^2)$, completing the first claim.

To establish monotonicity, we evaluate the marginal increase in $\Psi_G(p)$ for $p \in (0,1)$:
\[
\Psi_{G+1}(p) - \Psi_G(p) = p(1-p)^G + (1-p)p^G > 0,
\]
confirming $\Psi_G(p)$ strictly increases with $G$. For any $G_2 > G_1$, the first-order difference $\eta c_x (\Psi_{G_2}(p) - \Psi_{G_1}(p))$ is strictly positive and dominates the $O_{G_1, G_2}(\eta^2)$ remainder term for $\eta$ strictly less than some threshold $\eta_0(G_1, G_2) > 0$. Taking the minimum $\eta_0$ across all pairs in the finite set $\mathcal{G}$ establishes the strict ordering simultaneously. 

Finally, the second difference of $\Psi_G(p)$ is $-p^2(1-p)^G - (1-p)^2 p^G < 0$, confirming diminishing returns.
\end{proof}

\paragraph{Remark on \hyperref[ass:constant_cx]{Assumption~\ref{ass:constant_cx}}.}
\hyperref[ass:constant_cx]{A\ref{ass:constant_cx}} serves as an idealized first-order approximation to isolate the core mechanism: larger groups increase the mixed-reward probability $\Psi_G(p)$, thereby providing nonzero relative reward signals more frequently. In practice, GRPO normalizes advantages relative to the group (e.g., the positive advantage scales as $\sqrt{(G-s)/s}$), meaning the exact expected magnitude $c_x$ technically depends on $G$. However, this assumption keeps the mechanism analytically transparent and highlights a plausible first-order driver of the empirical monotonicity observed during training.

\paragraph{Remark on advantage-estimation variance.}
Our analysis isolates the \textit{signal-availability} mechanism, demonstrating that a larger \(G\) yields more frequent informative updates. We do not explicitly model the variance reduction in advantage estimation to avoid complex distributional assumptions. Therefore, this theorem formalizes one core driver of the avg@$k$ improvement, acknowledging that other statistical benefits also contribute.

\paragraph{Extension to Dataset-Level Notation.}
While \autoref{thm:avgk_monotone} characterizes the optimization dynamics of a single prompt, explaining the difficulty-dependent behavior requires dataset-level formalization. Let $p_{x,d}(G)$ denote the single-sample correctness probability on a test prompt $x \in \mathcal{D}$ for a policy trained on difficulty subset $d$ using rollout budget $G$. We define the expected average correctness ($\mu_d(G)$) and pass@$k$ over the dataset $\mathcal{D}$ as:
\[
\begin{aligned}
\mu_d(G) &:= \mathbb{E}_{x\sim\mathcal{D}}[p_{x,d}(G)], \\
\text{pass@k}_{\mathcal{D}}^{(d)}(G) &:= \mathbb{E}_{x\sim\mathcal{D}} \!\left[ 1-\bigl(1-p_{x,d}(G)\bigr)^k \right].
\end{aligned}
\]

\begin{theorem}[\textbf{Variance Penalty Decomposition of pass@$k$}]
\label{thm:passk_difficulty}
Fix $k\ge 2$. Let $f_k(p):=1-(1-p)^k$ and define the Jensen gap
\[
J_{k,d}(G)
:=
f_k\!\bigl(\mu_d(G)\bigr)
-
\text{pass@k}_{\mathcal{D}}^{(d)}(G).
\]
Then the following hold:

\begin{enumerate}
    \item $J_{k,d}(G)\ge 0$, with equality if and only if $p_{x,d}(G)$ is constant almost surely over $x\sim\mathcal{D}$.
    \item For any two rollout budgets $G_2>G_1$,
    \[
    \begin{aligned}
    &\text{pass@k}_{\mathcal{D}}^{(d)}(G_2)-\text{pass@k}_{\mathcal{D}}^{(d)}(G_1) \\
    &=
    \underbrace{
    f_k\!\bigl(\mu_d(G_2)\bigr)-f_k\!\bigl(\mu_d(G_1)\bigr)
    }_{\text{gain from higher }\text{avg@k}_{\mathcal{D}}} \\
    &\quad -
    \underbrace{
    \bigl(J_{k,d}(G_2)-J_{k,d}(G_1)\bigr)
    }_{\text{loss from test-set prompt heterogeneity}}.
    \end{aligned}
    \]
\end{enumerate}

Moreover, if $\sigma_d^2(G) := \operatorname{Var}_{x\sim\mathcal{D}}[p_{x,d}(G)]$, then
\[
\begin{aligned}
J_{k,d}(G)
&=
\frac{k(k-1)}{2}
\bigl(1-\mu_d(G)\bigr)^{k-2}
\sigma_d^2(G) \\
&\quad +
O\!\left(
\mathbb{E}_{x\sim\mathcal{D}}
\left[
|p_{x,d}(G)-\mu_d(G)|^3
\right]
\right).
\end{aligned}
\]
Hence, to second order, the penalty term is controlled by the cross-prompt dispersion of one-sample success probabilities on the test set.
\end{theorem}

\begin{proof}
For $k\ge 2$ and $p \in (0, 1)$,
\[
\begin{aligned}
f_k'(p) &= k(1-p)^{k-1} > 0, \\
f_k''(p) &= -k(k-1)(1-p)^{k-2} < 0,
\end{aligned}
\]
so $f_k$ is strictly increasing and strictly concave on $(0,1)$.

By definition, the expected mean is
\[
\mu_d(G)
=
\mathbb{E}_{x\sim\mathcal{D}}[p_{x,d}(G)],
\]
and the expected pass rate is
\[
\text{pass@k}_{\mathcal{D}}^{(d)}(G)
=
\mathbb{E}_{x\sim\mathcal{D}}[f_k(p_{x,d}(G))].
\]
Applying Jensen's inequality to the concave function $f_k$,
\[
\begin{aligned}
\text{pass@k}_{\mathcal{D}}^{(d)}(G)
&=
\mathbb{E}_{x\sim\mathcal{D}}[f_k(p_{x,d}(G))] \\
&\le
f_k\!\left(\mathbb{E}_{x\sim\mathcal{D}}[p_{x,d}(G)]\right) \\
&=
f_k\!\bigl(\mu_d(G)\bigr).
\end{aligned}
\]
Thus $J_{k,d}(G)\ge 0$. Since $f_k$ is strictly concave, equality holds if and only if $p_{x,d}(G)$ is constant almost surely.

Now rewrite
\[
\text{pass@k}_{\mathcal{D}}^{(d)}(G)=f_k\!\bigl(\mu_d(G)\bigr)-J_{k,d}(G).
\]
Subtracting the identities for $G_2$ and $G_1$ yields
\[
\begin{aligned}
&\text{pass@k}_{\mathcal{D}}^{(d)}(G_2)-\text{pass@k}_{\mathcal{D}}^{(d)}(G_1) \\
&=
\Bigl(
f_k\!\bigl(\mu_d(G_2)\bigr)-f_k\!\bigl(\mu_d(G_1)\bigr)
\Bigr) \\
&\quad -
\Bigl(
J_{k,d}(G_2)-J_{k,d}(G_1)
\Bigr),
\end{aligned}
\]
which gives the exact decomposition.

For the second-order approximation, apply Taylor's theorem around $\mu_d(G)$:
\[
\begin{aligned}
f_k(p_{x,d}(G))
&=
f_k\!\bigl(\mu_d(G)\bigr) \\
&\quad +
f_k'\!\bigl(\mu_d(G)\bigr)\bigl(p_{x,d}(G)-\mu_d(G)\bigr) \\
&\quad +
\frac{1}{2}
f_k''\!\bigl(\mu_d(G)\bigr)\bigl(p_{x,d}(G)-\mu_d(G)\bigr)^2 \\
&\quad +
O(|p_{x,d}(G)-\mu_d(G)|^3) .
\end{aligned}
\]
Taking the expectation over $x\sim\mathcal{D}$, the linear term vanishes because $\mathbb{E}_{x\sim\mathcal{D}}[p_{x,d}(G)-\mu_d(G)]=0$. Hence
\[
\begin{aligned}
\mathbb{E}_{x\sim\mathcal{D}}[f_k(p_{x,d}(G))]
&=
f_k\!\bigl(\mu_d(G)\bigr) \\
&\quad +
\frac{1}{2}
f_k''\!\bigl(\mu_d(G)\bigr)\sigma_d^2(G) \\
&\quad +
O\Bigl( \mathbb{E}_{x\sim\mathcal{D}} \bigl[|p_{x,d}(G) \\
&\qquad\qquad-\mu_d(G)|^3 \bigr] \Bigr).
\end{aligned}
\]
Rearranging gives
\[
\begin{aligned}
J_{k,d}(G)
&=
-\frac{1}{2}
f_k''\!\bigl(\mu_d(G)\bigr)\sigma_d^2(G) \\
&\quad +
O\Bigl( \mathbb{E}_{x\sim\mathcal{D}} \bigl[ |p_{x,d}(G)-\mu_d(G)|^3 \bigr] \Bigr),
\end{aligned}
\]
and substituting $f_k''(\mu)=-k(k-1)(1-\mu)^{k-2}$ completes the proof.
\end{proof}

\paragraph{Connection to Empirical Regimes.}
Rather than positing a universal law applicable to all training environments, \autoref{thm:passk_difficulty} provides a rigorous analytical lens to interpret the specific empirical observations reported in \hyperref[sec:adaptive_rollout]{Section~\ref{sec:adaptive_rollout}}. Specifically, while the theorem formally establishes the mathematical \textit{mechanism}---that the dynamics of $\text{pass@k}_{\mathcal{D}}$ are entirely governed by a trade-off between the mean improvement gain ($\Delta f_k$) and the cross-prompt variance penalty ($\Delta J_{k,d}$)---the actual increase of this variance ($\sigma_d^2(G)$) is an \textit{empirical characteristic} dictated by the intrinsic difficulty of the dataset and specific training dynamics. Consequently, while these exact behavioral patterns may not identically manifest in every setup, bridging this theoretical mechanism with our specific findings allows us to formally characterize the three observed regimes as follows:
\[
\begin{aligned}
&\text{\textbf{Easy}:} \quad \Delta_G J_{k,\text{Easy}}(G) > \Delta_G f_{k,\text{Easy}}(G) \\
& \Longrightarrow\quad \text{pass@k}_{\mathcal{D}}^{(\text{Easy})}(G+1) < \text{pass@k}_{\mathcal{D}}^{(\text{Easy})}(G), \\[2mm]
&\text{\textbf{Hard}:} \quad \Delta_G J_{k,\text{Hard}}(G) < \Delta_G f_{k,\text{Hard}}(G) \\
& \Longrightarrow\quad \text{pass@k}_{\mathcal{D}}^{(\text{Hard})}(G+1) > \text{pass@k}_{\mathcal{D}}^{(\text{Hard})}(G), \\[2mm]
&\text{\textbf{Medium}:} \quad \Delta_G J_{k,\text{Medium}}(G) \\
&\qquad\qquad \text{crosses}\ \Delta_G f_{k,\text{Medium}}(G) \\
& \Longrightarrow\quad \text{pass@k}_{\mathcal{D}}^{(\text{Medium})}(G)\ \text{is non-monotonic}.
\end{aligned}
\]
where $\Delta_G J_{k,d}(G) = J_{k,d}(G+1)-J_{k,d}(G)$ and $\Delta_G f_{k,d}(G) = f_k(\mu_d(G+1))-f_k(\mu_d(G))$.

Analytically, the second-order approximation reveals that the penalty term $J_{k,d}(G)$ is directly proportional to the cross-prompt variance $\sigma_d^2(G)$. This theoretically explains the paradoxical degradation in $\text{pass@k}_{\mathcal{D}}$ observed when training on the \textbf{Easy} dataset. When the rollout budget $G$ is large, the policy over-exploits specific, easily solvable templates rather than learning generalizable reasoning skills. Evaluated on a general test set, this causes severe polarization: the success probability $p_{x,d}(G)$ converges to $1$ for familiar easy prompts but remains near $0$ for unseen or harder problems. Consequently, the surge in the variance penalty ($\Delta_G J_{k,\text{Easy}}(G)$) outweighs the marginal gains in mean performance ($\Delta_G f_{k,\text{Easy}}(G)$), resulting in a decrease in the overall $\text{pass@k}_{\mathcal{D}}$.

Conversely, on the \textbf{Hard} dataset, valid reasoning paths are scarce. An increased rollout budget promotes broader exploration, enabling the model to internalize more generalizable reasoning patterns. This uniformly improves the test-set mean $\mu_d(G)$ across prompts, maintaining low cross-prompt variance and keeping the associated penalty minimal.

\section{Experimental Details of \hyperref[sec:adaptive_rollout]{Section~\ref{sec:adaptive_rollout}}}
\label{app:experimental_details_adaptive}
We train Qwen2.5-3B-Instruct \citep{qwen2.5} on the MATH \citep{hendrycks2021measuring} dataset, which we partition into Easy, Medium, and Hard subsets based on the base model's empirical accuracy. Specifically, for each problem, we sample 12 independent solutions from the base model and categorize problems according to the number of correct samples. Problems for which all 12 samples are either entirely correct or entirely incorrect are excluded, as they provide no informative signal for distinguishing difficulty. Among the remaining problems, those with 1--4 correct samples are categorized as \textit{Hard}, 5--8 as \textit{Medium}, and 9--11 as \textit{Easy}. This results in 2,217 Easy, 1,433 Medium, and 1,712 Hard problems. To ensure a balanced training distribution, we subsample each subset to 1,400 problems. The model was trained for 5 epochs on each difficulty-specific dataset.

\begin{table}[t]
\centering
\setlength{\tabcolsep}{6pt}
\renewcommand{\arraystretch}{1.1}
\begin{tabular}{lc}
\toprule
\textbf{Training Configuration} & \textbf{Wall Clock (Hours)} \\
\midrule
Easy / $G=4$  & 5.0  \\
Easy / $G=8$  & 9.2  \\
Easy / $G=16$ & 16.4 \\
\hdashline
Medium / $G=4$  & 5.4  \\
Medium / $G=8$  & 9.4  \\
Medium / $G=16$ & 16.6 \\
\hdashline
Hard / $G=4$  & 5.5  \\
Hard / $G=8$  & 9.9  \\
Hard / $G=16$ & 17.5 \\
\bottomrule
\end{tabular}
\caption{Wall Clock Time for each difficulty subset and rollout budget.}
\label{tab:adaptive_training_cost}
\end{table}

Using GRPO \citep{shao2024deepseekmath}, we conduct a total of 9 training runs, applying rollout budgets of $G \in \{4, 8, 16\}$ for each difficulty subset. Our implementation is based on a modified version of the GRPO-Zero repository\footnote{\url{https://github.com/policy-gradient/GRPO-Zero}}, which adopts a token-level policy gradient without KL regularization or clipping. We use a binary verifiable reward, assigning 1 to correct final answers and 0 otherwise. During training, all samples are generated with temperature 1.0 and a maximum generation length of 1024 tokens. We fix all optimization hyperparameters, including a batch size of 256, while adjusting the number of problems per batch (64, 32, 16) to realize group sizes of 4, 8, and 16, respectively. All experiments are conducted on a single A100 GPU. The wall-clock times for each training configuration are reported in \autoref{tab:adaptive_training_cost}.

We evaluate the resulting 9 models on three benchmarks: MATH500, AIME25, and AIME24. For each problem, we generate 256 reasoning paths in parallel. We compute pass@$k$ for $k \in \{1,2,4,8,16,32,64,128,256\}$, and report results for each model trained under different difficulty subsets and rollout budgets (\autoref{fig:appendix_passk_3x3}). In addition, we aggregate results across the three benchmarks and report the averaged avg@256 and pass@256, as shown in \autoref{fig:adaptive_results}.

\begin{figure*}[t]
\centering
\begin{subfigure}[b]{0.32\textwidth}
    \centering
    \includegraphics[width=\linewidth]{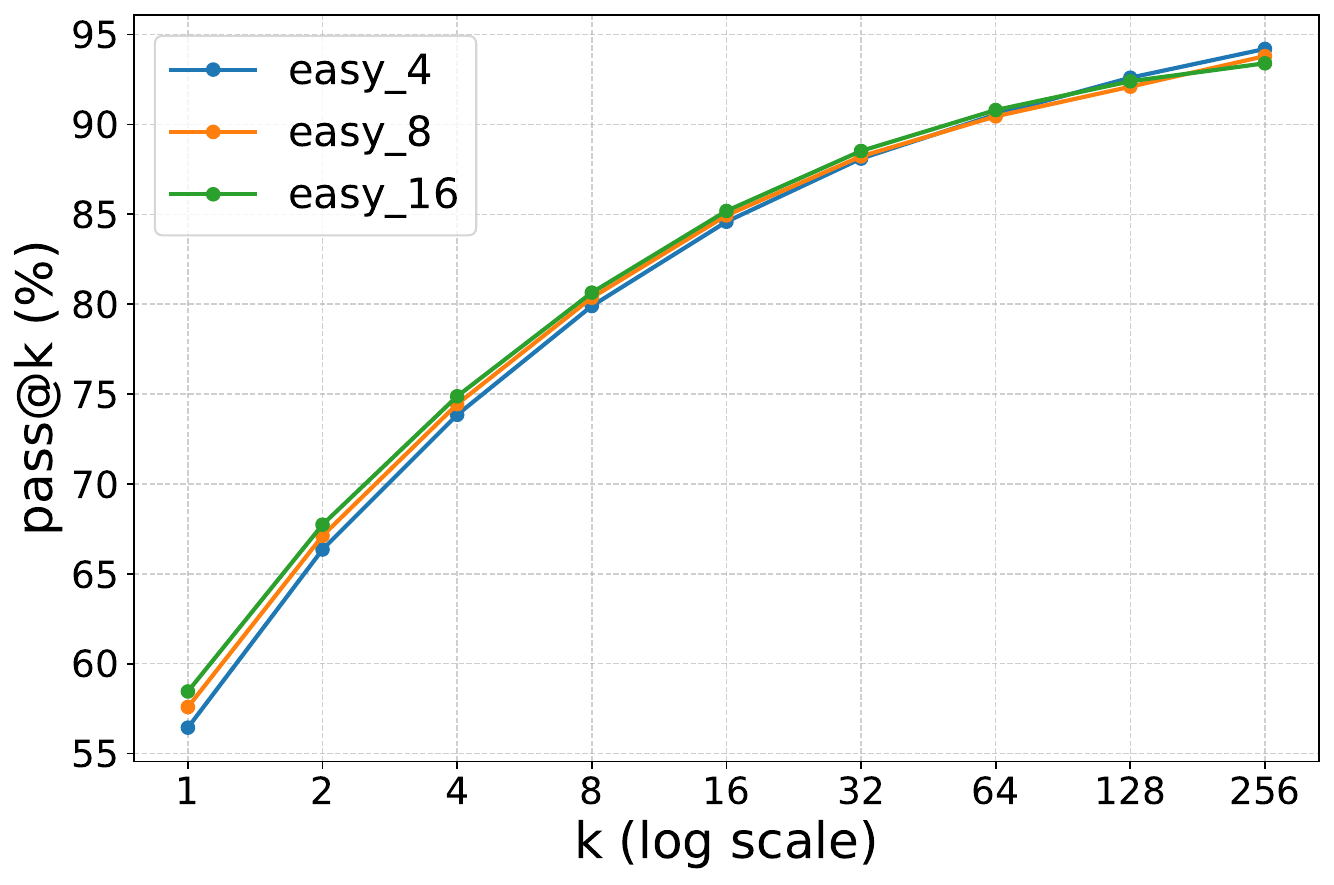}
    \caption{Easy / MATH500}
\end{subfigure}
\hfill
\begin{subfigure}[b]{0.32\textwidth}
    \centering
    \includegraphics[width=\linewidth]{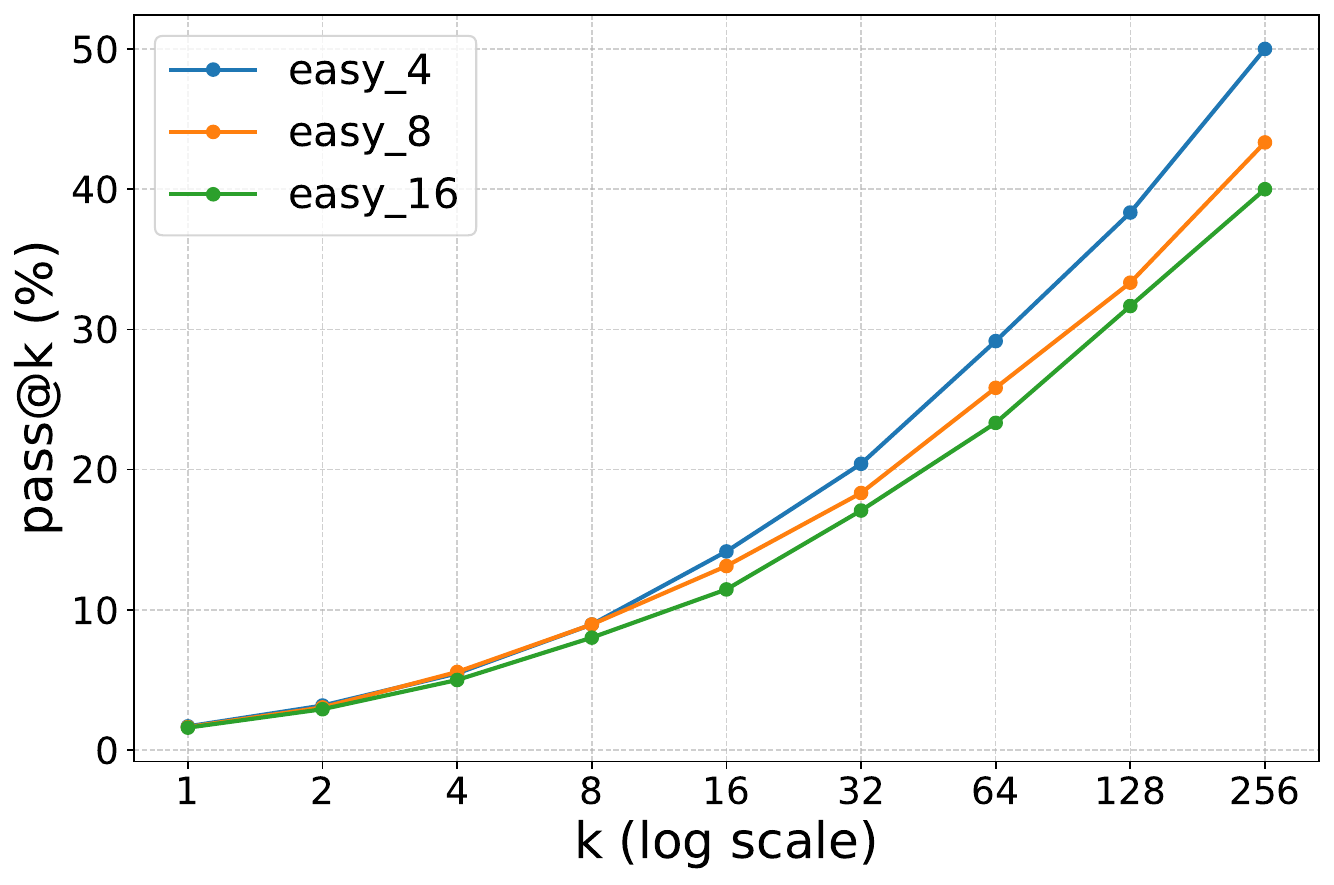}
    \caption{Easy / AIME25}
\end{subfigure}
\hfill
\begin{subfigure}[b]{0.32\textwidth}
    \centering
    \includegraphics[width=\linewidth]{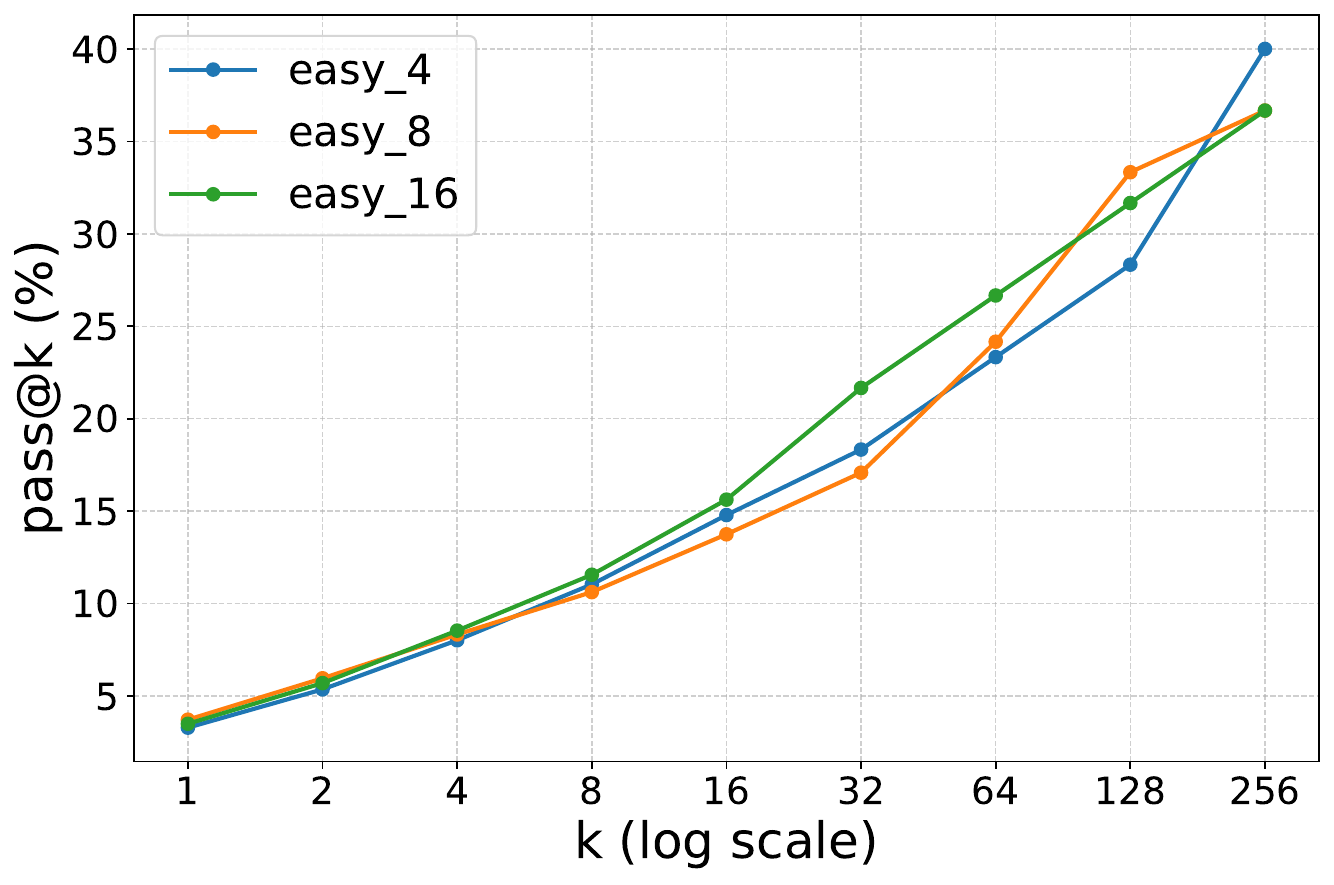}
    \caption{Easy / AIME24}
\end{subfigure}

\vspace{4pt}

\begin{subfigure}[b]{0.32\textwidth}
    \centering
    \includegraphics[width=\linewidth]{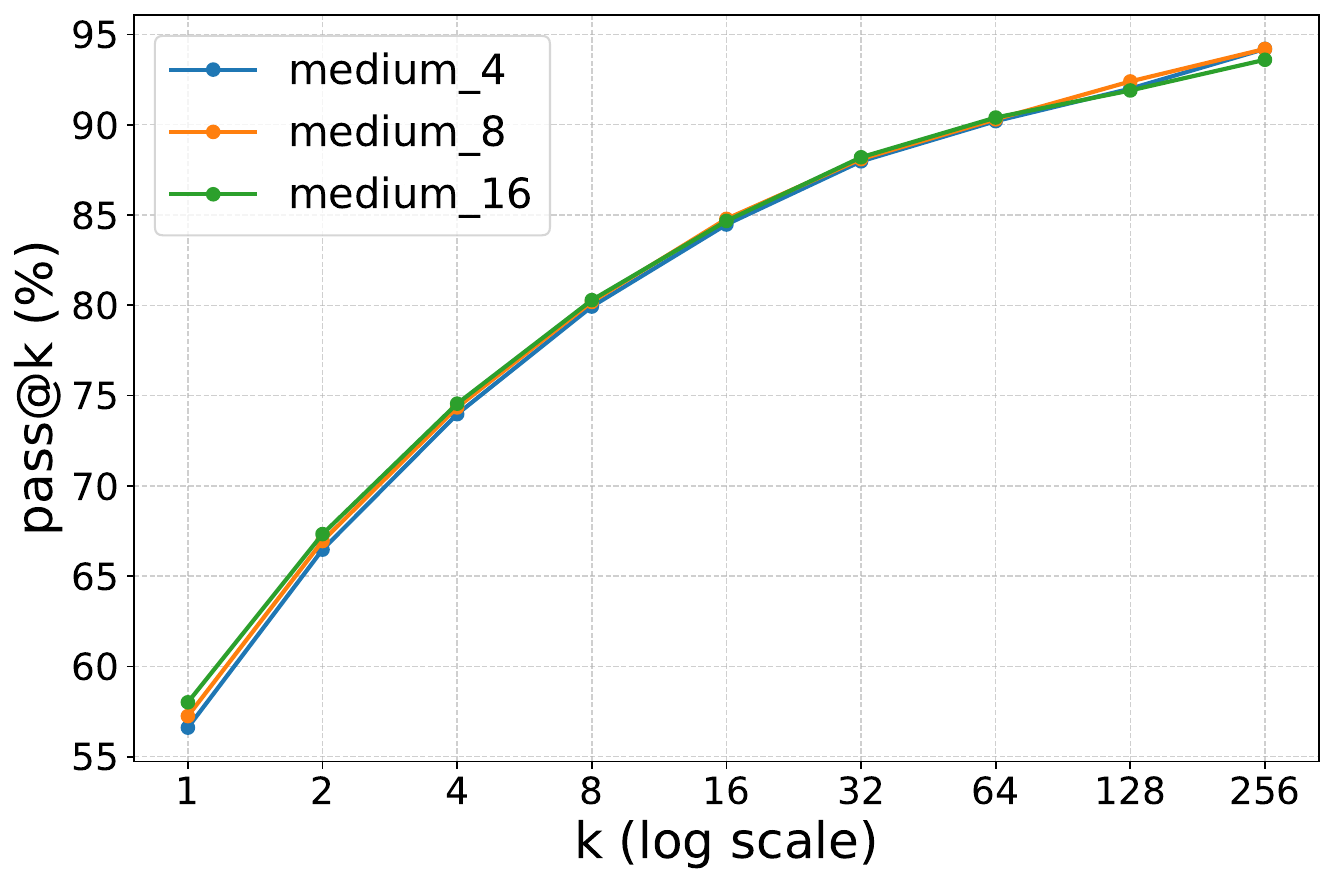}
    \caption{Medium / MATH500}
\end{subfigure}
\hfill
\begin{subfigure}[b]{0.32\textwidth}
    \centering
    \includegraphics[width=\linewidth]{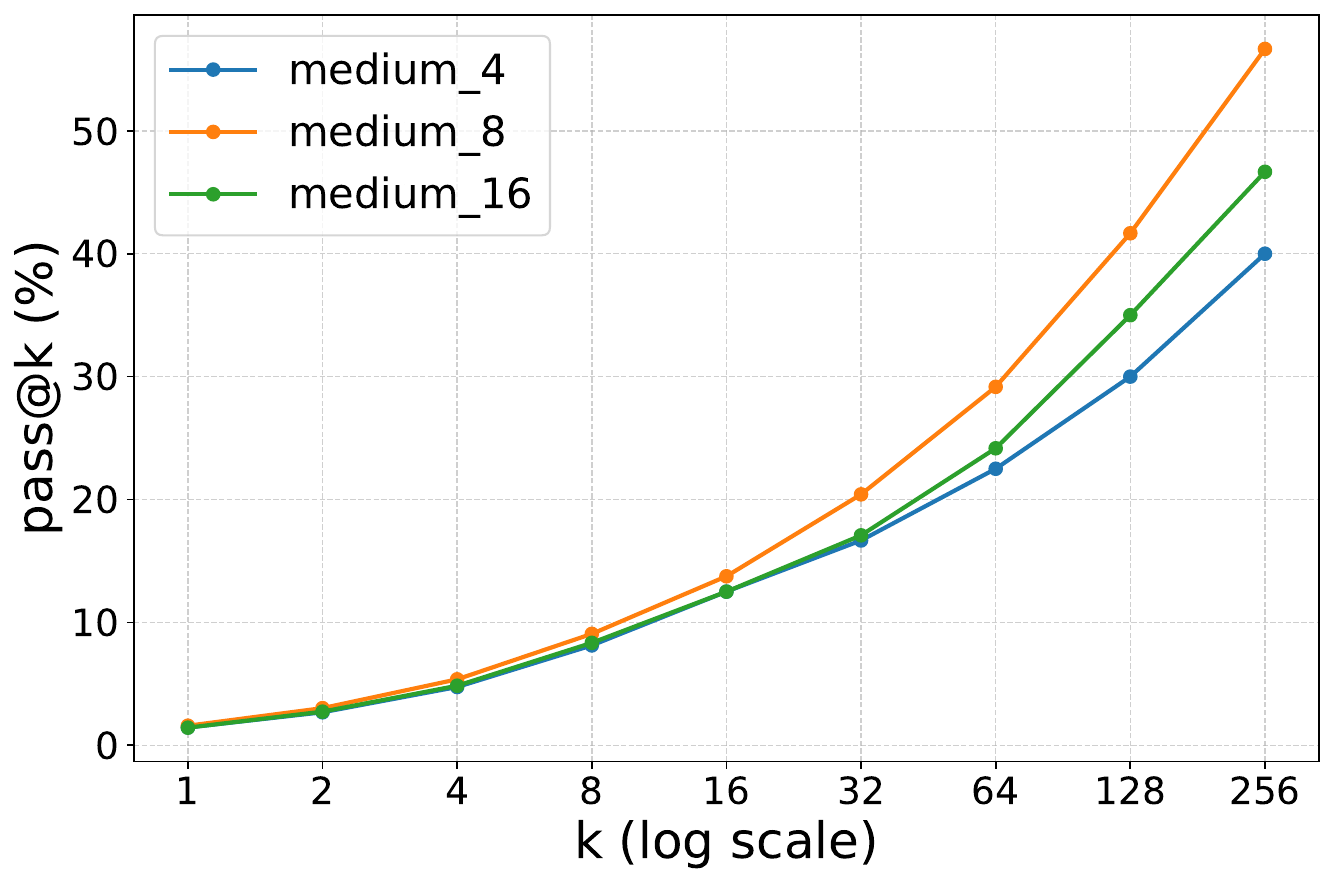}
    \caption{Medium / AIME25}
\end{subfigure}
\hfill
\begin{subfigure}[b]{0.32\textwidth}
    \centering
    \includegraphics[width=\linewidth]{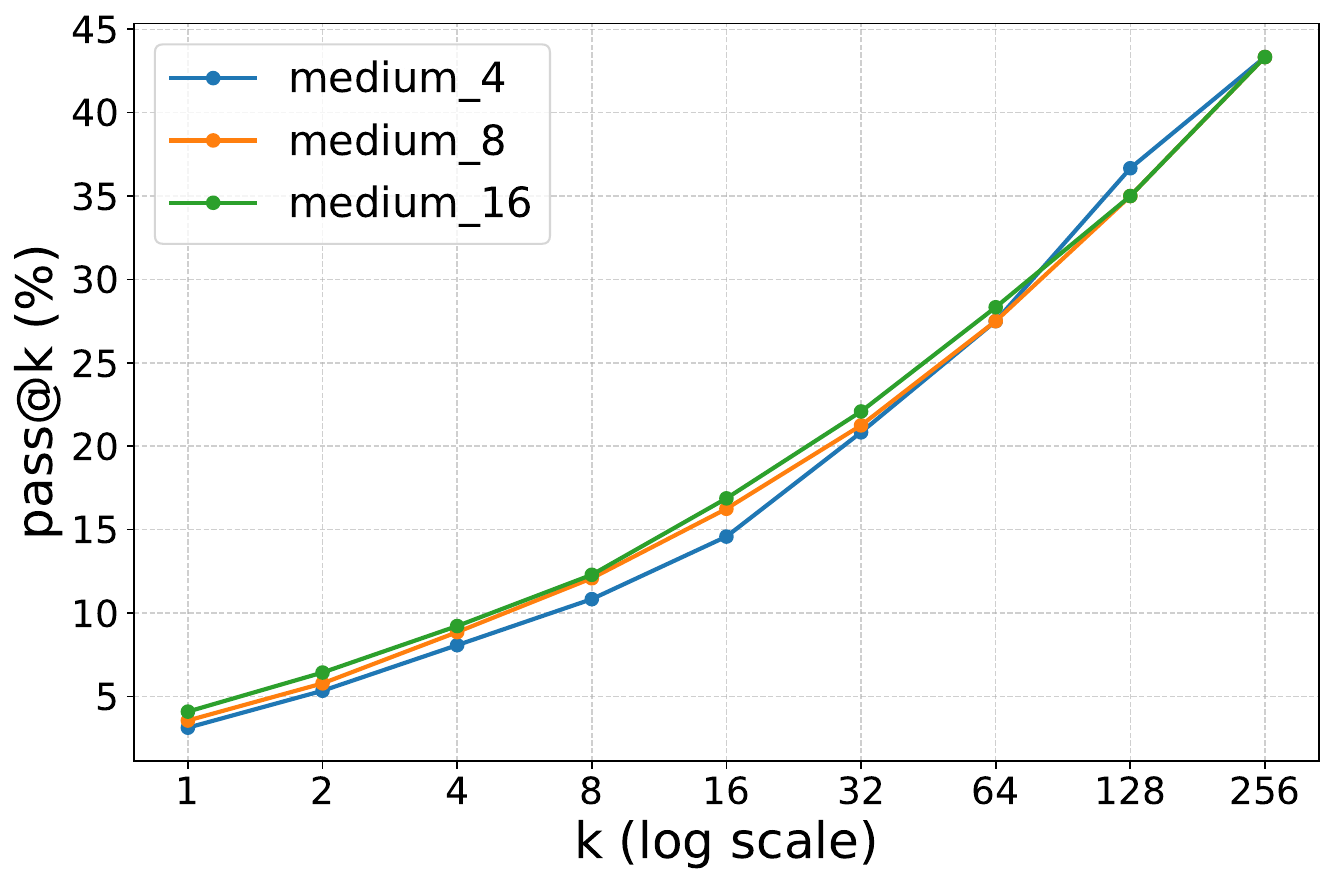}
    \caption{Medium / AIME24}
\end{subfigure}

\vspace{4pt}

\begin{subfigure}[b]{0.32\textwidth}
    \centering
    \includegraphics[width=\linewidth]{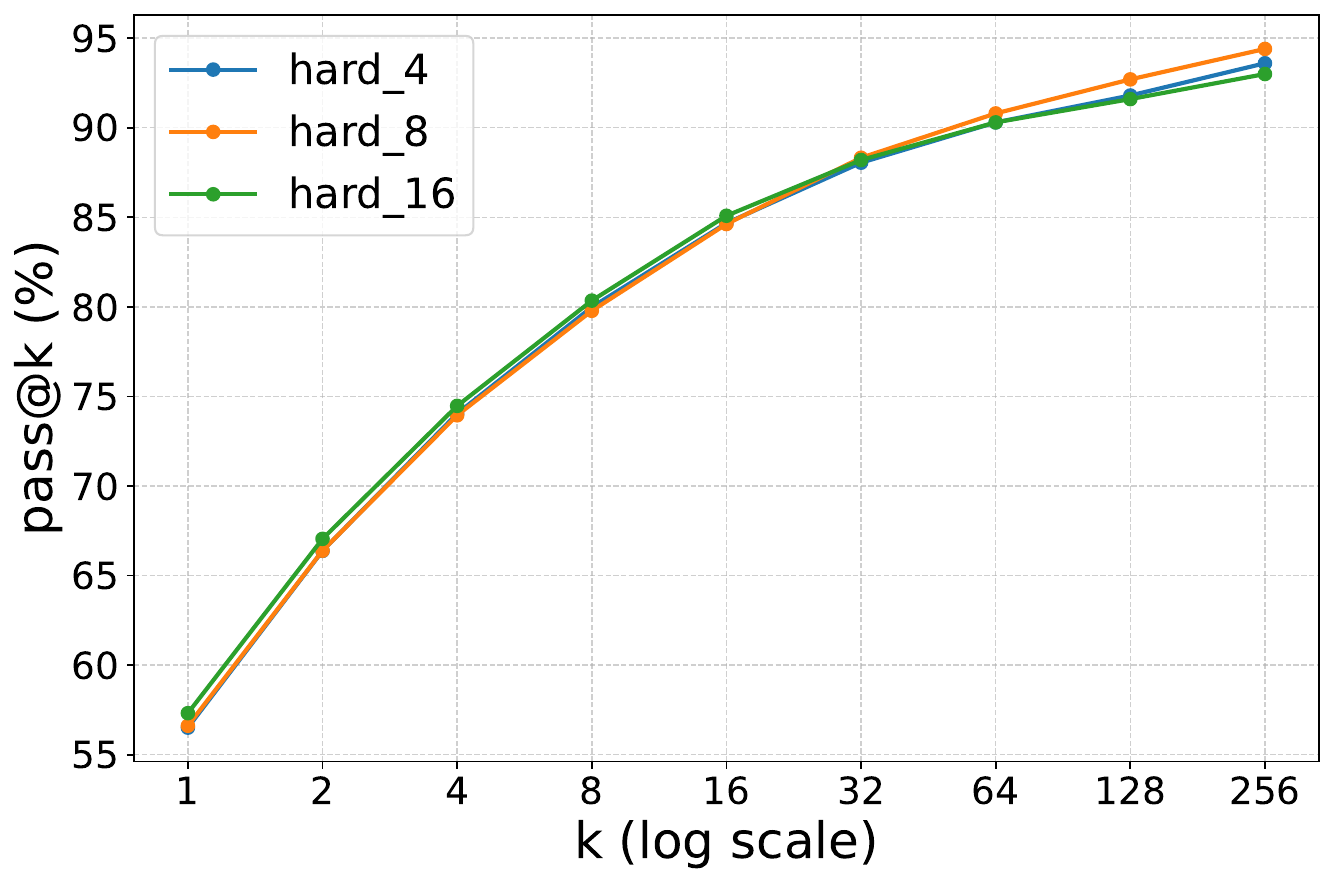}
    \caption{Hard / MATH500}
\end{subfigure}
\hfill
\begin{subfigure}[b]{0.32\textwidth}
    \centering
    \includegraphics[width=\linewidth]{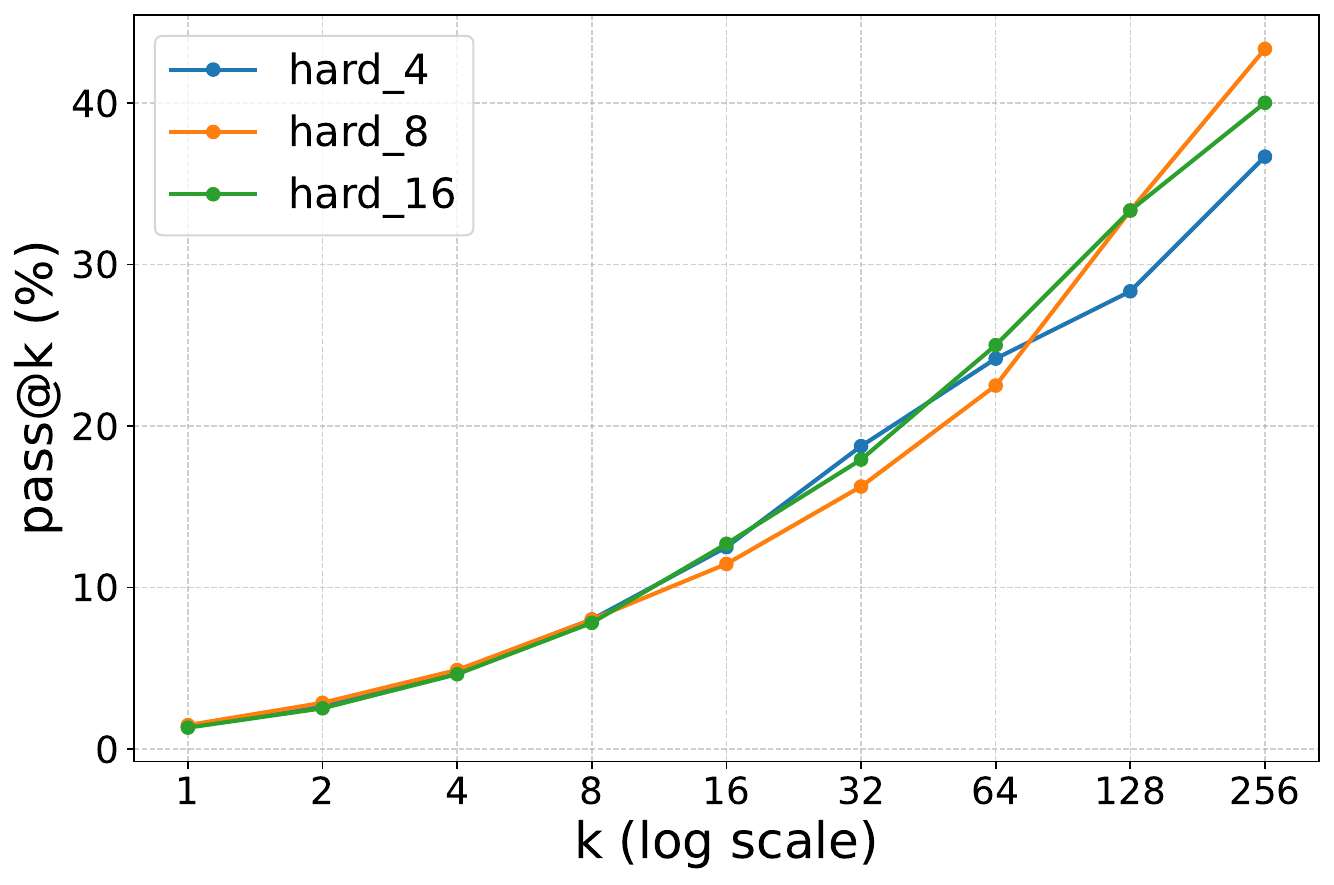}
    \caption{Hard / AIME25}
\end{subfigure}
\hfill
\begin{subfigure}[b]{0.32\textwidth}
    \centering
    \includegraphics[width=\linewidth]{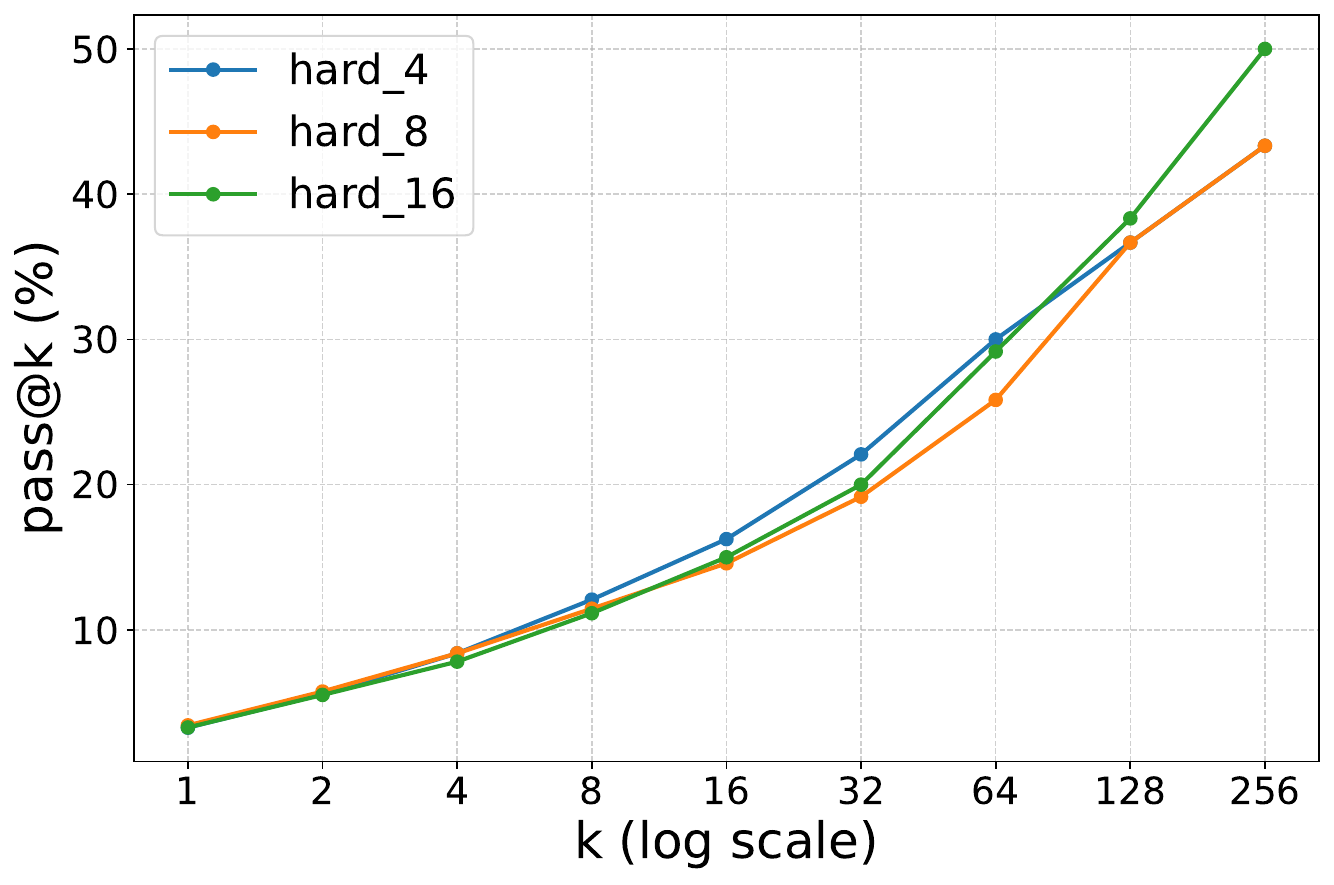}
    \caption{Hard / AIME24}
\end{subfigure}

\caption{pass@$k$ performance on evaluation benchmarks for models trained on different difficulty subsets. Labels such as \texttt{easy\_8} denote models trained on the Easy subset with group size $G=8$.}
\label{fig:appendix_passk_3x3}
\end{figure*}

\section{Details on Tree-Structured Rollout (\hyperref[sec:tree_vs_parallel]{Section \ref{sec:tree_vs_parallel}})}
\label{app:tree_vs_parallel}

\subsection{Tree Algorithm Details}
\label{app:tree_algorithm}
\hyperref[alg:tree_rollout]{Algorithm \ref{alg:tree_rollout}} provides the detailed procedure of the two-phase tree rollout strategy.

\begin{algorithm}[t]
\caption{Tree Rollout}
\label{alg:tree_rollout}
\renewcommand{\algorithmicrequire}{\textbf{Input:}}
\renewcommand{\algorithmicensure}{\textbf{Output:}}
\begin{algorithmic}
\REQUIRE Policy $\pi_\theta$, Prompt $x$, Base rollouts $N$, Forking points $K$, Branch rollouts $B$, Selection \texttt{method}.
\ENSURE Trajectory set $\mathcal{T}$
\STATE Initialize $\mathcal{R} \leftarrow \emptyset$, $\mathcal{T} \leftarrow \emptyset$

\STATE \textbf{Phase 1: Base Rollouts}
\FOR{$i = 1$ to $N$}
    \STATE Sample base trajectory $\tau^{(i)} \sim \pi_\theta(\cdot \mid x)$
    \STATE $\mathcal{R} \leftarrow \mathcal{R} \cup \{\tau^{(i)}\}$
    \STATE $\mathcal{T} \leftarrow \mathcal{T} \cup \{\tau^{(i)}\}$
\ENDFOR

\STATE \textbf{Phase 2: Branch Rollouts}
\FOR{each $\tau \in \mathcal{R}$}
    \STATE Select $K$ forking points $\mathcal{F} \subset \{1, \dots, |\tau|\}$ via \autoref{alg:forking_selection}
    \FOR{each $f \in \mathcal{F}$}
        \STATE Extract trajectory prefix $x_f \leftarrow [x; \tau_{1:f}]$
        \FOR{$j = 1$ to $B$}
            \STATE Sample continuation $\tau^{\prime(j)}_f \sim \pi_\theta(\cdot \mid x_f)$
            \STATE $\tilde{\tau}^{(j)}_f \leftarrow [\tau_{1:f}; \tau^{\prime(j)}_f]$
            \STATE $\mathcal{T} \leftarrow \mathcal{T} \cup \{\tilde{\tau}^{(j)}_f\}$
        \ENDFOR
    \ENDFOR
\ENDFOR
\end{algorithmic}
\end{algorithm}

\subsection{Experimental Details of \hyperref[sec:tree_or_parallel_analysis]{Section \ref{sec:tree_or_parallel_analysis}}}
\label{app:experimental_details_tree_parallel}

To evaluate cost-efficiency during inference, we measure the PassRate---the probability of discovering at least one correct solution---against token consumption on the AIME25 benchmark using Qwen2.5-3B-Instruct. We compare standard parallel sampling with two tree-structured variants: \texttt{fixed-seg} (equidistant branching) and \texttt{random} (uniform random branching). Detailed procedures for these forking strategies are provided in \hyperref[app:forking_point_algorithm]{Appendix \ref{app:forking_point_algorithm}}.

\paragraph{Formal Definition of PassRate}
For parallel sampling, the model generates $M$ independent trajectories, $\mathcal{T}_{\text{parallel}} = \{\tau^{(1)}, \dots, \tau^{(M)}\}$. The PassRate is defined as:
$$ \text{PassRate}_{\text{parallel}} = \mathbb{I}\left(\sum_{i=1}^{M} r(\tau^{(i)}) \ge 1\right) $$
where $r(\cdot) \in \{0, 1\}$ is the verifiable binary reward. Token consumption scales linearly with $M$.

For tree-structured rollouts, generating $N$ base rollouts, $K$ forking points per base, and $B$ branches per forking point yields a set of leaf trajectories $\mathcal{T}_{\text{tree}}$. The total number of terminal paths is $N(1 + KB)$. The PassRate is evaluated over these leaves:
$$ \text{PassRate}_{\text{tree}} = \mathbb{I}\left(\sum_{\tau \in \mathcal{T}_{\text{tree}}} r(\tau) \ge 1\right) $$
Unlike parallel sampling, tree rollouts consume tokens sub-linearly relative to the number of terminal paths, as context prefixes up to each forking point are computed only once and shared across branches.

\paragraph{Configurations}
To analyze PassRate against token consumption (\autoref{fig:tree_parallel}), we varied the computational budget for both methods. For parallel sampling, we evaluated independent rollout counts of $M \in \{4, 8, 16, 32\}$. For tree rollouts, we fixed the branching parameters at $K=2$ and $B=2$, and scaled the base rollouts $N \in \{2, 4, 8\}$. These configuration tuples $(N, K, B)$ yield 10, 20, and 40 terminal paths, respectively, enabling a direct comparison of exploration efficiency relative to token expenditure. Detailed forking point selection algorithms are provided in \hyperref[app:forking_point_algorithm]{Appendix~\ref{app:forking_point_algorithm}}.

\section{Details on Forking Point Selection Strategies (\hyperref[sec:optimal_forking]{Section \ref{sec:optimal_forking}})}
\label{app:optimal_forking}

\begin{algorithm}[t]
\caption{Forking Points Selection Algorithms}
\label{alg:forking_selection}
\renewcommand{\algorithmicrequire}{\textbf{Input:}}
\renewcommand{\algorithmicensure}{\textbf{Output:}}
\begin{algorithmic}
\REQUIRE Trajectory $\tau$, Number of forking points $K$, Selection \texttt{method}, Token entropies $\{H_t\}_{t=1}^{|\tau|}$, Parsed sentences $\mathcal{S} = \{S_1, \dots, S_M\}$, Sentence entropies $\{H^{S}_m\}_{m=1}^M$, Parsed steps $\mathcal{Z} = \{Z_1, \dots, Z_L\}$, Step-level FCI scores $\{y_l\}_{l=1}^L$.
\ENSURE Forking points $\mathcal{F} \subset \{1, \dots, |\tau|\}$, where $|\mathcal{F}| = K$

\IF{\texttt{method} is \texttt{random}}
    \STATE Sample $\mathcal{F} \subset \{1, \dots, |\tau|\}$ uniformly at random
\ELSIF{\texttt{method} is \texttt{fixed-seg}}
    \STATE $\mathcal{F} \leftarrow \left\{ \left\lfloor k \cdot \frac{|\tau|}{K+1} \right\rfloor \;\middle|\; k = 1, \dots, K \right\}$
\ELSIF{\texttt{method} is \texttt{ATB}}
    \STATE $\mathcal{C} \leftarrow \{ l \in \{1, \dots, L\} \mid $
    \STATE \quad $ y_l \ge \text{Quantile}(y_1, \dots, y_L, 0.8) \}$
    \STATE $ \mathcal{L}^* \leftarrow$ the $K$ smallest indices (earliest steps) from $\mathcal{C}$
    \STATE $\mathcal{F} \leftarrow \{\nu_l \mid l \in \mathcal{L}^*\}$
\ELSIF{\texttt{method} is \texttt{tok-entropy}}
    \STATE $\mathcal{F} \leftarrow \mathop{\mathrm{arg\,top-}K}\limits_{t \in \{1, \dots, |\tau|\}} H_t$
\ELSIF{\texttt{method} is \texttt{sent-entropy}}
    \STATE $\mathcal{M}^* \leftarrow \mathop{\mathrm{arg\,top-}K}\limits_{m \in \{1, \dots, M\}} H^{S}_m$
    \STATE $\mathcal{F} \leftarrow \{\mu_m \mid m \in \mathcal{M}^*\}$
\ENDIF

\RETURN $\mathcal{F}$
\end{algorithmic}
\end{algorithm}

\subsection{Forking Point Selection Algorithms}
\label{app:forking_point_algorithm}

We detail the five forking point selection algorithms in \hyperref[alg:forking_selection]{Algorithm~\ref{alg:forking_selection}}, which is utilized in Phase 2 Branch Rollouts. Let a trajectory be denoted as a sequence of tokens $\tau = (y_1, y_2, \dots, y_{|\tau|})$. For entropy-based methods, let $H_t$ denote the token-level entropy at step $t$, which is computed using the language model's probability distribution over the vocabulary $V$:
$$ H_t = -\sum_{v \in V} P(v \mid y_{<t}) \log P(v \mid y_{<t}) $$

For the sentence-level method, we parse the trajectory into $M$ sentences, $\mathcal{S} = \{S_1, S_2, \dots, S_M\}$, by aligning character-level sentence boundaries detected via PySBD \cite{sadvilkar-neumann-2020-pysbd} with the tokenizer's offset mapping. Each sentence $S_m$ represents a contiguous span of token indices, and $\mu_m$ denotes the start token index of $S_m$. The sentence entropy $H^{S}_m$ is defined as the average token entropy within that sentence to mitigate length bias:
$$ H^{S}_m = \frac{1}{|S_m|} \sum_{t \in S_m} H_t $$

Attention-based Tree Branching (\texttt{ATB}) operates on the premise that attention scores serve as meaningful metrics to identify important reasoning behaviors. To leverage this insight, the trajectory is first segmented into $L$ discrete steps (e.g., delimited by line breaks), denoted as $\mathcal{Z} = \{Z_1, \dots, Z_L\}$, where $\nu_l$ represents the start token index of step $Z_l$. ATB then computes a Forward Context Influence (FCI) score $y_l$ for each step, quantifying its influence on subsequent tokens by aggregating attention weights. When selecting forking points, ATB first isolates a candidate set consisting of the top 20\% of steps with the highest FCI scores. From this set, it selects the $K$ earliest steps as the final forking points.

Lastly, the operator $\mathop{\mathrm{arg\,top-}K}$ denotes the extraction of the subset of indices that correspond to the $K$ highest values of a given metric.

\subsection{Quantitative and Qualitative Analysis of Forking-Point Localization}
\label{app:localiztion_and_case_study}
As illustrated in \autoref{fig:localization}, high-entropy tokens tend to densely cluster within narrow, highly uncertain segments of the reasoning trajectory. Consequently, the search budget is monopolized by repeated resampling at a single localized point, which severely limits the structural reach of the search tree. 

To quantitatively characterize the localization phenomenon, we measure the distribution of selected forking points using three metrics: Nearest Neighbor Distance (NND), the average distance to the closest forking point; Mean Pairwise Distance (MPD), the average pairwise distance among all selected points; and Window-based Clustering Ratio (WCR@$P$), the percentage of forking-point pairs whose distance falls within $P\%$ of the trajectory length. Lower NND and MPD and higher WCR indicate stronger localization. 

\begin{table}[t]
    \centering
    \small
    \begin{tabular}{@{}lccc@{}}
        \toprule
        \textbf{Metric} 
        & \texttt{random}
        & \texttt{sent-entropy}
        & \texttt{tok-entropy} \\
        \midrule
        NND $\downarrow$    & 0.157 & 0.130 & \textbf{0.106} \\
        MPD $\downarrow$    & 0.337 & 0.296 & \textbf{0.251} \\
        WCR@5 $\uparrow$    & 8.9   & 12.9  & \textbf{23.6}  \\
        WCR@10 $\uparrow$   & 16.4  & 25.3  & \textbf{35.3}  \\
        \bottomrule
    \end{tabular}
    \caption{Quantitative comparison of forking-point localization across different selection strategies. Arrows indicate stronger localization.}
    \label{tab:localization_quantitative}
\end{table}

As shown in \autoref{tab:localization_quantitative}, \texttt{tok-entropy} exhibits the strongest localization, with the lowest NND (0.106) and MPD (0.251) and the highest WCR@5 (23.6\%) and WCR@10 (35.3\%). In comparison, \texttt{sent-entropy} increases NND and MPD to 0.130 and 0.296, respectively, while reducing WCR@5 and WCR@10 to 12.9\% and 25.3\%, quantitatively demonstrating that sentence-level entropy mitigates the localization of forking points.

To further observe how the selected tokens differ across forking strategies, we conducted a case study comparing \texttt{tok-entropy}, \texttt{sent-entropy}, and \texttt{ATB}. \autoref{fig:forking_point_case_study} visualizes the forking points ($K=4$) selected by each method along a single, fixed reasoning path. The \texttt{tok-entropy} method suffers from localization, trapping the selected points within a narrow segment. In contrast, \texttt{sent-entropy} still targets the model's uncertain regions but distributes the forking points much more evenly, operating at a broader semantic level. Meanwhile, \texttt{ATB} selects steps receiving the highest attention weights; while this can be meaningful for dividing logical reasoning steps, it operates independently of the model's uncertainty. As shown in \autoref{fig:forking_point_case_study}, \texttt{ATB} often selects highly deterministic points to branch from. Performing additional sampling at these points yields low SibDiv and, consequently, a limited PassRate, as reflected in \autoref{tab:forking_point_results}. Based on these observations, \texttt{sent-entropy} proves to be the most effective forking strategy by successfully leveraging uncertainty signals while resolving the localization issue.

\begin{figure*}[t!]
    \centering
    \begin{minipage}{0.48\linewidth}
        \centering
        \includegraphics[width=\linewidth]{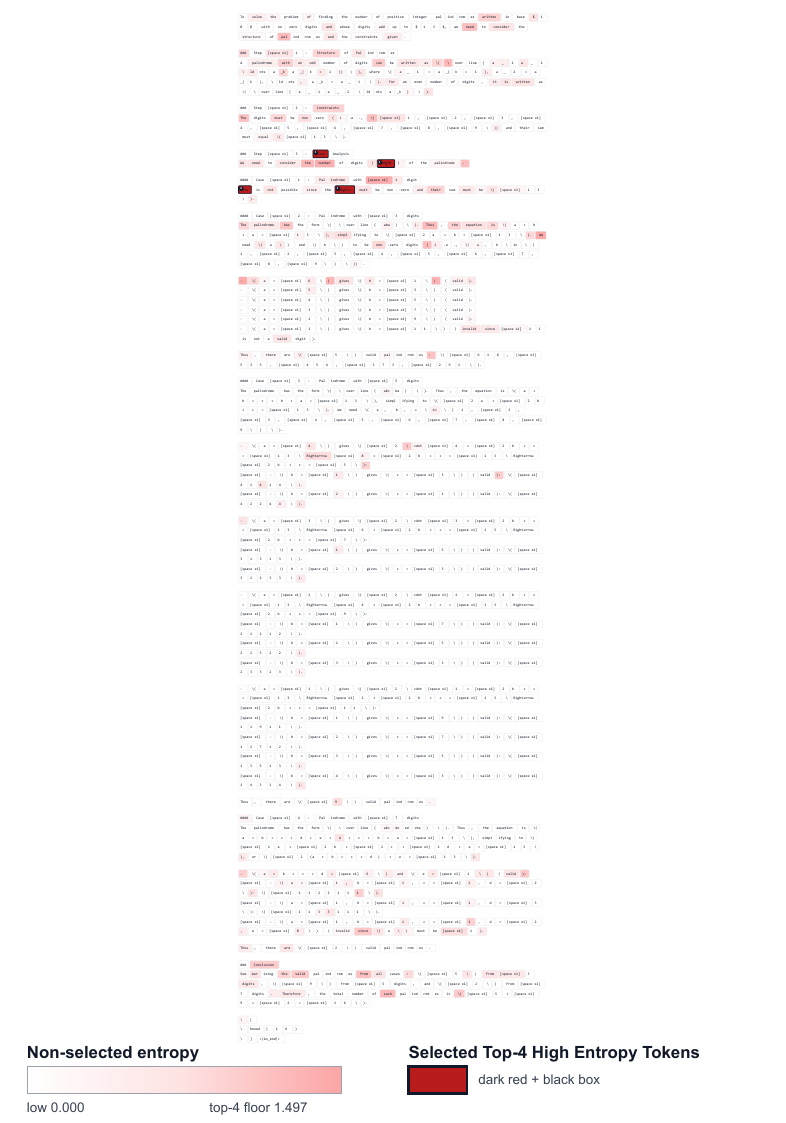}
    \end{minipage}
    \hfill
    \begin{minipage}{0.48\linewidth}
        \centering
        \includegraphics[width=\linewidth]{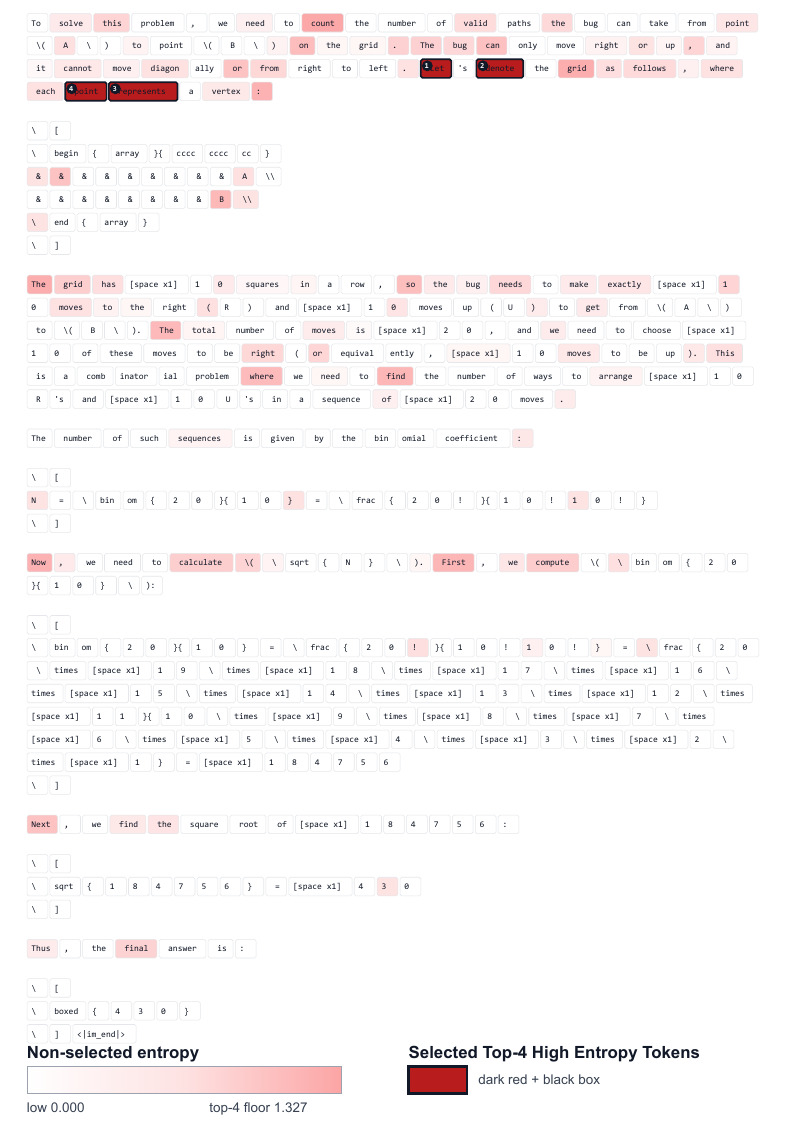}
    \end{minipage}

    \vspace{0.6em}

    \begin{minipage}{0.48\linewidth}
        \centering
        \includegraphics[width=\linewidth]{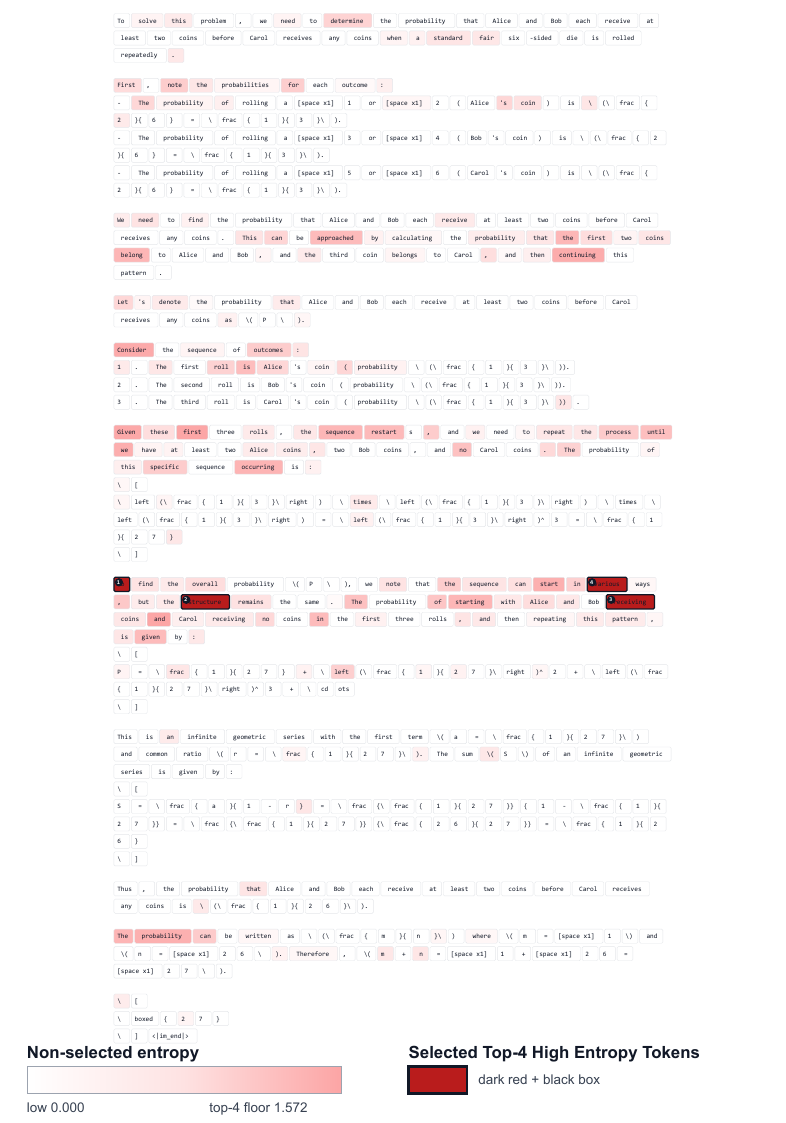}
    \end{minipage}
    \hfill
    \begin{minipage}{0.48\linewidth}
        \centering
        \includegraphics[width=\linewidth]{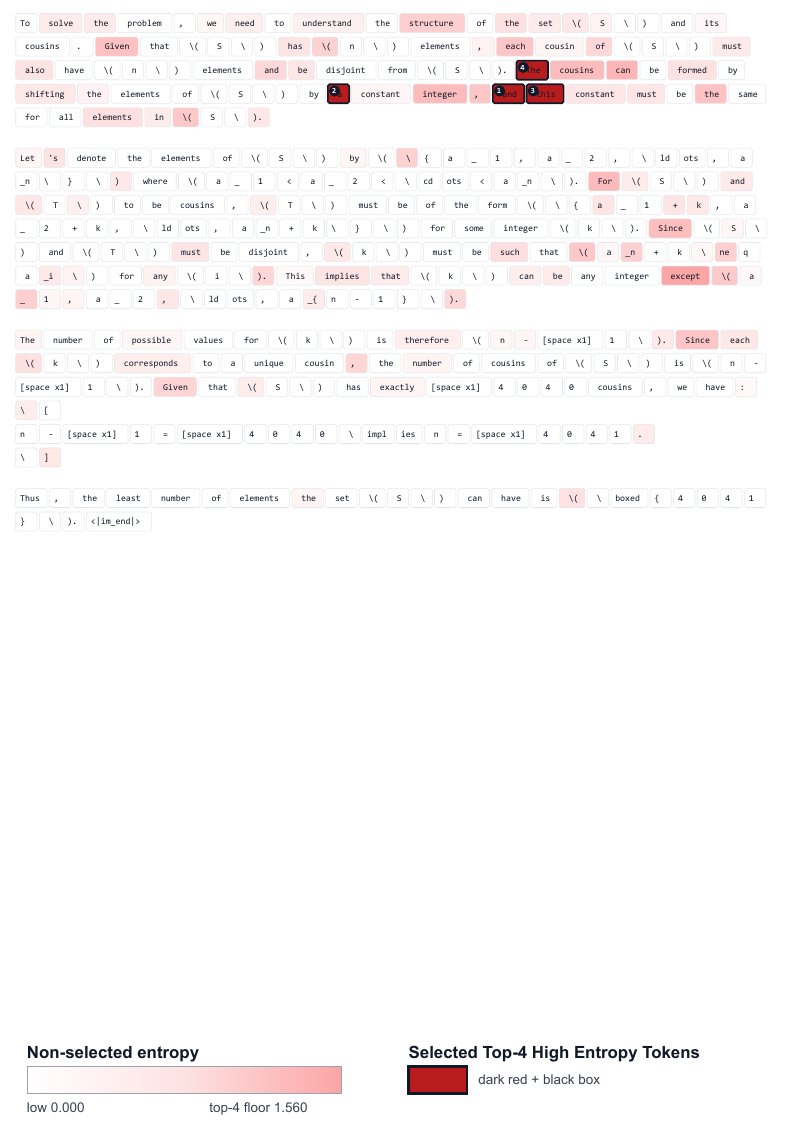}
    \end{minipage}

    \caption{Visualization of token-entropy \textit{localization}. Dark red boxed tokens indicate the selected top-4 high-entropy forking points. These points densely cluster within a narrow segment, monopolizing the search budget and limiting structural diversity.}
    \label{fig:localization}
\end{figure*}

\begin{figure*}[t!]
    \centering
    \begin{minipage}{0.48\linewidth}
        \centering
        \includegraphics[width=\linewidth]{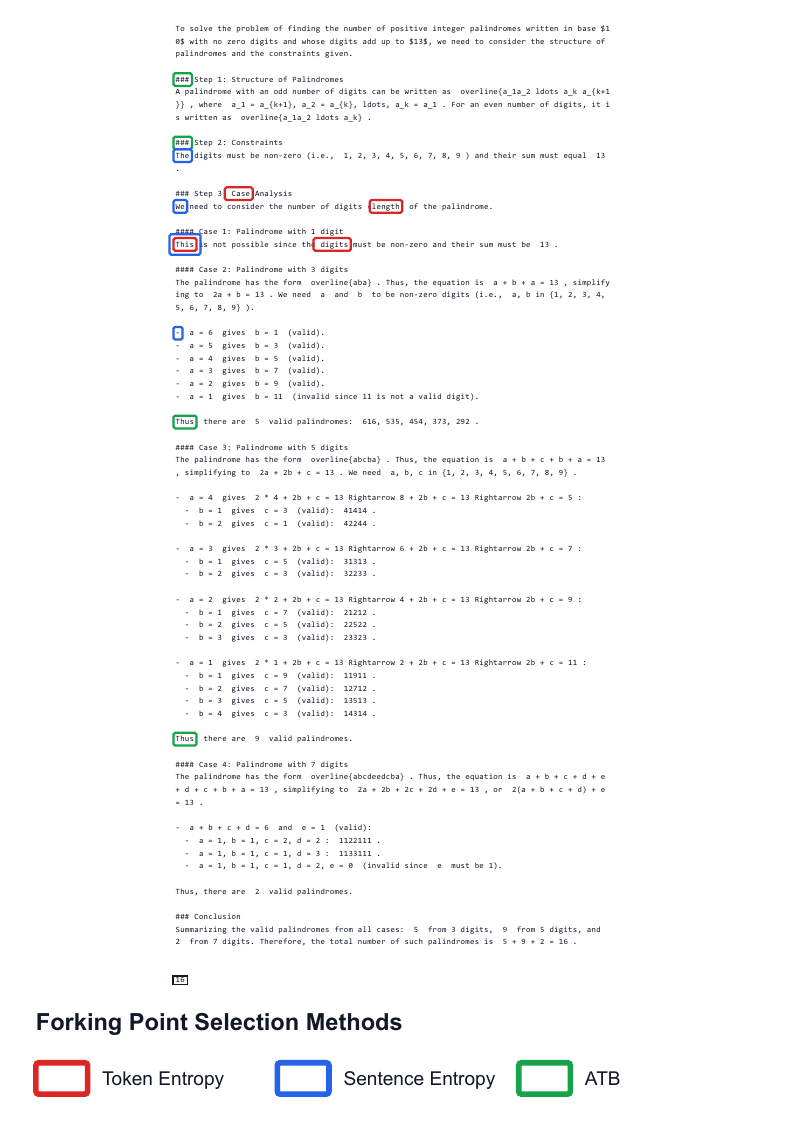}
    \end{minipage}
    \hfill
    \begin{minipage}{0.48\linewidth}
        \centering
        \includegraphics[width=\linewidth]{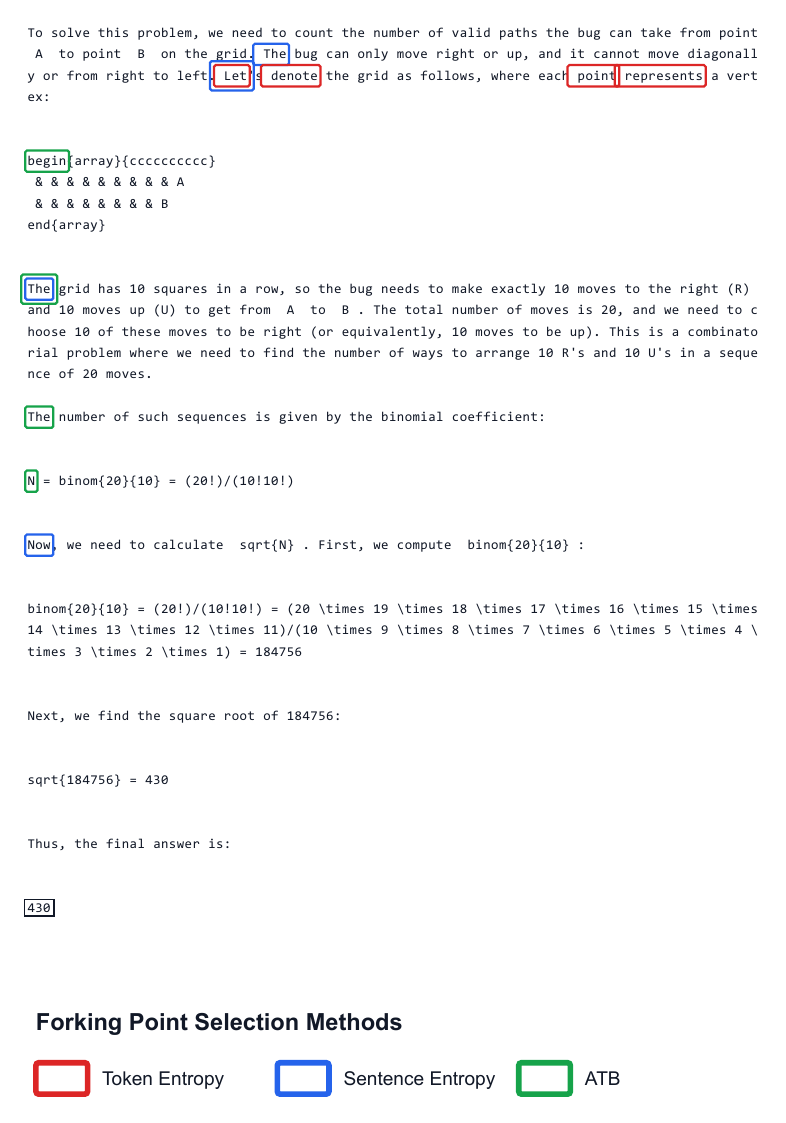}
    \end{minipage}

    \vspace{0.6em}

    \begin{minipage}{0.48\linewidth}
        \centering
        \includegraphics[width=\linewidth]{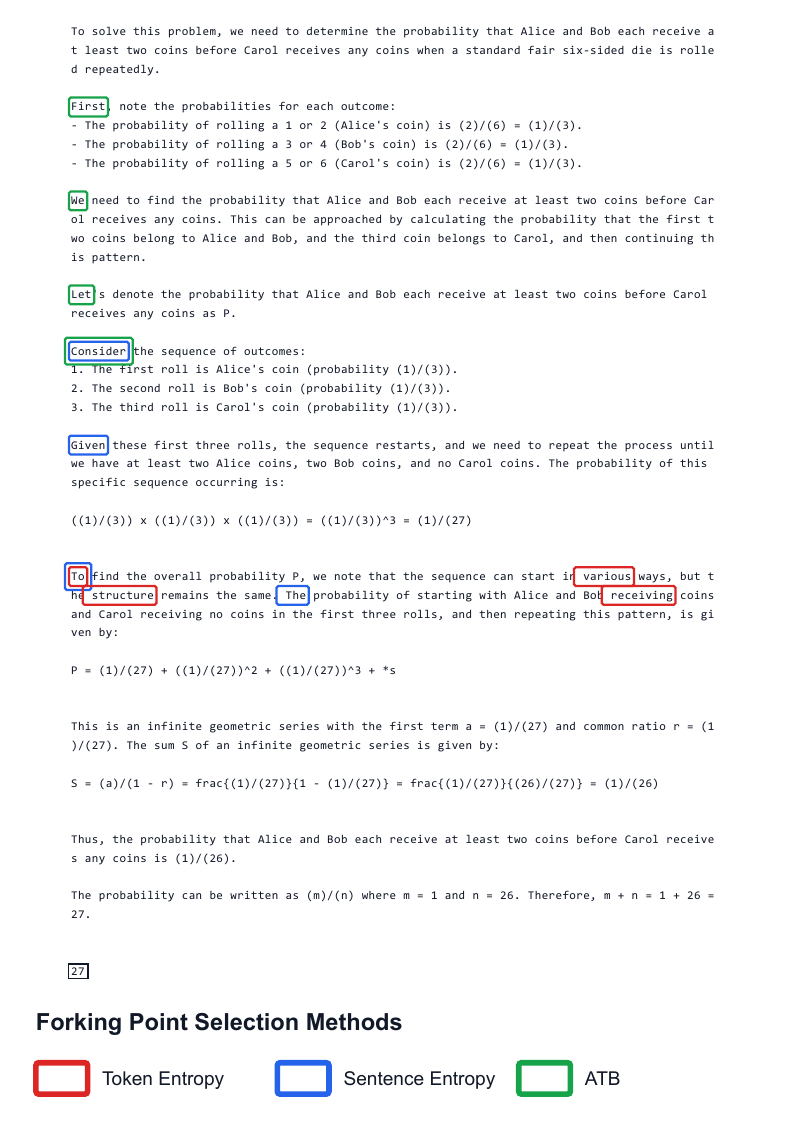}
    \end{minipage}
    \hfill
    \begin{minipage}{0.48\linewidth}
        \centering
        \includegraphics[width=\linewidth]{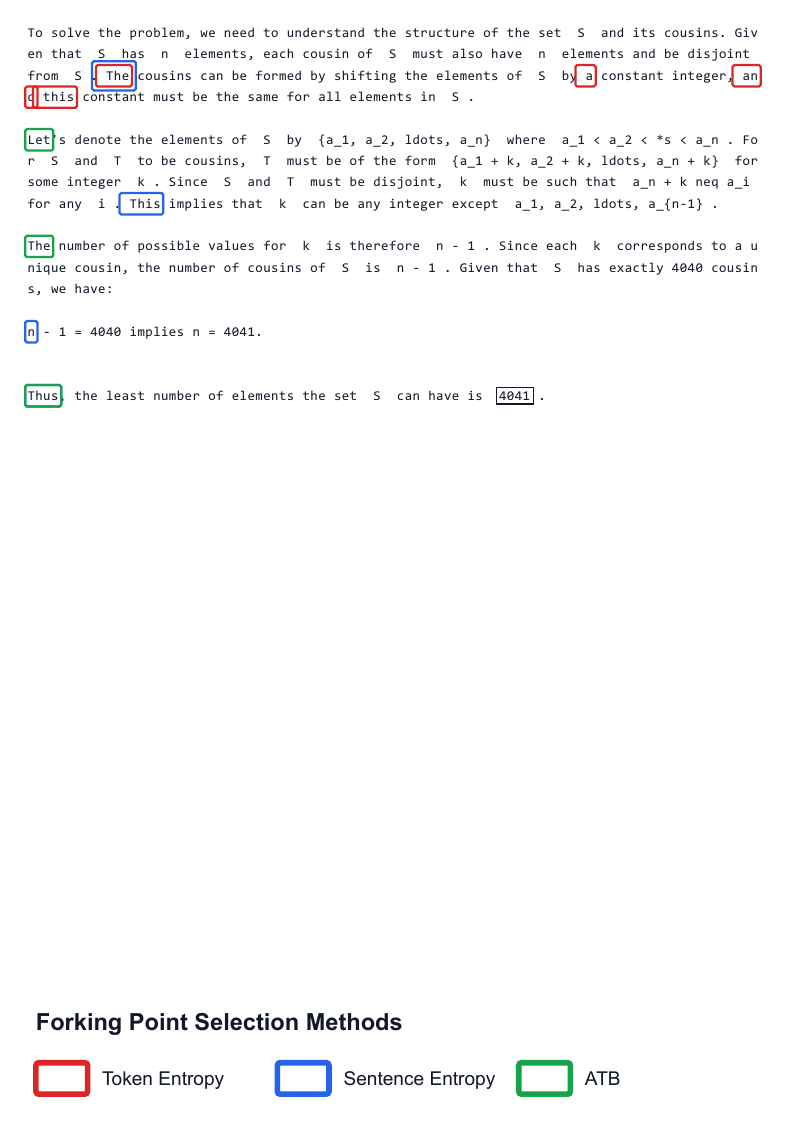}
    \end{minipage}

    \caption{Case studies comparing forking point selections across \texttt{tok-entropy}, \texttt{sent-entropy}, and \texttt{ATB}. Colored boxes indicate the tokens selected by each method, with overlapping boxes denoting agreement between methods. Notably, while \texttt{tok-entropy} suffers from localization, \texttt{sent-entropy} distributes points more evenly across broader semantic levels, whereas \texttt{ATB} often targets deterministic points regardless of uncertainty.}
    \label{fig:forking_point_case_study}
\end{figure*}

\subsection{Formal Definitions of SibDiv}
\label{app:diversity_metrics}

Because recent studies emphasize semantic diversity for effective exploration \citep{jiang2025selective, yao2025diversity}, it is critical to rigorously quantify the divergence of the explored reasoning paths discussed in \hyperref[sec:forking_metric]{Section \ref{sec:forking_metric}}. To this end, we model the generated reasoning tree as a collection of text segments, referred to as \textit{blocks}. Formally, a \textit{block} is defined as a contiguous sequence of reasoning spanning from the root to the first forking point, between two consecutive forking points, or from a final forking point to a leaf. 

Let $\mathcal{F}$ denote the set of all forking points and $\mathcal{B}$ denote the set of all blocks within a single generated reasoning tree. We extract a dense semantic representation $\mathbf{e}_b$ for each block $b \in \mathcal{B}$ utilizing the \texttt{gte-large-en-v1.5} model \citep{zhang-etal-2024-mgte}. Sibling Diversity (SibDiv) is formulated as the complement of the average pairwise cosine similarity among sibling blocks:

\begin{equation}
    \textit{SibDiv} = 1 - \frac{1}{|\mathcal{F}|} \sum_{f \in \mathcal{F}} \frac{\sum_{b_i, b_j \in B(f), \, i < j} \cos(\mathbf{e}_{b_i}, \mathbf{e}_{b_j})}{\binom{|B(f)|}{2}}
\end{equation}

where $B(f) \subset \mathcal{B}$ represents the specific subset of child blocks branching directly from a given forking point $f$, and $\cos(\cdot, \cdot)$ denotes the cosine similarity function. 

A higher SibDiv score indicates that the model successfully generates semantically distinct continuations from its forking points. By aggregating this localized diversity, SibDiv effectively captures the degree of semantic exploration and reasoning diversity across the entire tree.

\subsection{Experimental Details of \hyperref[sec:forking_analysis]{Section~\ref{sec:forking_analysis}}}
\label{app:experimental_details_forking_points}
We evaluated each forking strategy by generating reasoning trees under a fixed inference budget, measuring SibDiv and PassRate. Specifically, we conducted inference on the AIME26 benchmark using the Qwen2.5-3B-Instruct model. To isolate the effect of the forking point selection strategy from the inherent randomness of language model generation, we generated and fixed a single base rollout for each problem. We then applied each selection method to identify $K=4$ distinct forking points along this fixed base trajectory. From each identified forking point, we generated $B=4$ additional branch rollouts. This procedure ensures a strictly controlled environment, yielding exactly 17 terminal leaves (1 base + 16 new branches) per problem across all methods. Finally, we calculated the PassRate and SibDiv across these 17 leaves and averaged the metrics over the entire dataset. To ensure the robustness of our findings, we repeated this entire experimental procedure five independent times and report the final averaged results in \autoref{tab:forking_point_results}.

\begin{table}[t]
    \centering
    \small
    \renewcommand{\arraystretch}{1.2}
    \begin{tabular}{lcc}
        \toprule
        \textbf{Method} 
        & \textbf{PassRate}
        & \textbf{SibDiv} \\
        \midrule
        \texttt{random}                       & 10.0 & 0.0796 \\
        \texttt{fixed-seg}                    & 13.3 & 0.0853 \\
        \texttt{ATB}                          & 14.7 & 0.0874 \\
        \texttt{tok-entropy}                  & 12.0 & \textbf{0.1065} \\
        \texttt{tok-entropy} w/ 5\% dist.     & 14.0 & 0.1041 \\
        \texttt{tok-entropy} w/ 10\% dist.    & \underline{15.3} & 0.1007 \\
        \texttt{sent-entropy}                 & \textbf{17.3} & \underline{0.1049} \\
        \bottomrule
    \end{tabular}
    \caption{Comparison of forking-point selection strategies, including distance-forced token-entropy baselines.}
    \label{tab:distance_forced_tok_entropy}
\end{table}

\subsection{Analyzing the Benefits of Sentence-Level Entropy}
While \texttt{sent-entropy} substantially improves PassRate over \texttt{tok-entropy}, it remains unclear whether this gain arises simply from mitigating localization or from identifying more semantically meaningful decision points. To disentangle these effects, we introduce distance-forced \texttt{tok-entropy} baselines that preserve the original token-level entropy signal while explicitly preventing nearby forking points. Starting from candidates ranked by token entropy, we greedily select forking points and skip any candidate whose distance from an already selected point is within a predefined fraction of the total trajectory length. We consider minimum-distance thresholds of 5\% and 10\%.

As shown in \autoref{tab:distance_forced_tok_entropy}, explicitly enforcing a minimum distance between token-level forking points improves PassRate from 12.0 for standard \texttt{tok-entropy} to 14.0 and 15.3 with the 5\% and 10\% constraints, respectively. This confirms that localization itself is a limiting factor for token-level entropy-based branching. However, both distance-forced variants still underperform \texttt{sent-entropy}, which achieves a PassRate of 17.3. Since these variants mitigate localization while retaining the same token-level entropy signal, the remaining gap suggests that the benefit of \texttt{sent-entropy} cannot be attributed solely to distributing forking points more broadly. Rather, sentence-level entropy provides a more effective signal for identifying semantically meaningful decision points for branching.

\section{Algorithmic and Formal Details of DATPO}
\label{app:datpo_details}

This section provides the formal definitions of the advantage components and the complete training procedure for Difficulty-Adaptive Sentence-entropy Guided Tree-Structured Policy Optimization (DATPO), expanding upon the methodology outlined in \hyperref[sec:methodology]{Section~\ref{sec:methodology}}.

\subsection{Formal Definitions of Value Estimation and Sibling Diversity term}

\paragraph{Monte Carlo State Value ($\hat{V}_{MC}$)}
For a state $s$ corresponding to a specific forking point, let $\mathcal{T}(s)$ denote the set of all terminal blocks descending from $s$. The Monte Carlo state value is formally defined as the expected verifiable reward over these terminal paths:
\begin{equation}
    \hat{V}_{MC}(s) = \frac{1}{|\mathcal{T}(s)|} \sum_{\tau \in \mathcal{T}(s)} r(\tau).
\end{equation}
For a terminal state $s_{\text{end}}$, the set of descending paths is empty, strictly enforcing $\hat{V}_{MC}(s_{\text{end}}) = 0$.

\paragraph{Block-Level Sibling Diversity ($\text{Div}_{\text{sib}}(b)$)}
Let $P(b)$ denote the forking point from which block $b$ originates, and let $\mathcal{B}_{\text{sib}}(b)$ denote the set of all sibling blocks branching from $P(b)$. The sibling-diversity term for a specific block $b$ is computed as its average semantic distance to its siblings:
\begin{equation}
    \text{Div}_{\text{sib}}(b) = 1 - \frac{\sum_{b^\prime \in \mathcal{B}_{\text{sib}}(b) \setminus \{b\}} \cos(\mathbf{e}_b, \mathbf{e}_{b^\prime})}{|\mathcal{B}_{\text{sib}}(b)| - 1}
    \label{equ:sibling_diversity},
\end{equation}
where $\mathbf{e}_b$ is the embedding of block $b$. We define $\text{Div}_{\text{sib}}(b)=0$ when $b$ has no valid sibling block.

While the evaluation metric \textit{SibDiv} (defined in \hyperref[app:diversity_metrics]{Appendix \ref{app:diversity_metrics}}) aggregates pairwise similarities across \textit{all} forking points to report a single tree-level scalar, the reward term $\text{Div}_{\text{sib}}(b)$ is evaluated at the individual block level to provide the advantage augmentation.

\subsection{DATPO Objective Function}
DATPO employs a token-level policy gradient objective mapped across the tree topology. Given a set of generated blocks $\mathcal{B}$ for a prompt $q$, the objective is defined as:
\begin{equation}
\begin{split}
\mathcal{J}_{\mathrm{DATPO}}(\theta) &= \mathbb{E}_{q \sim \mathcal{Q}, \mathcal{B} \sim \pi_{\theta_{\mathrm{old}}}} \Bigg[ \frac{1}{\sum_{b \in \mathcal{B}} |b|} \sum_{b \in \mathcal{B}} \sum_{t=1}^{|b|} \\
&\quad \min \Big( \rho_{b,t}(\theta) \hat{A}(b), \mathrm{clip}\big(\rho_{b,t}(\theta), \\
&\qquad\quad 1-\epsilon, 1+\epsilon\big) \hat{A}(b) \Big) \Bigg],
\end{split}
\label{equ:objective_DATPO}
\end{equation}
where $|b|$ is the sequence length of block $b$, $\rho_{b,t}(\theta) = \frac{\pi_\theta(y_t \mid x_{<t})}{\pi_{\theta_{\mathrm{old}}}(y_t \mid x_{<t})}$ is the importance ratio, and $\hat{A}(b)$ is the block-level augmented advantage.

\subsection{Complete DATPO Algorithm}

\hyperref[alg:datpo]{Algorithm \ref{alg:datpo}} details the end-to-end training procedure. By extending the two-phase rollout strategy into the training framework, the process is structured into four distinct phases: base rollouts, adaptive branch rollouts, block-level advantage estimation, and policy update.

\begin{algorithm}[t!]
\caption{\textbf{DATPO}}
\label{alg:datpo}
\renewcommand{\algorithmicrequire}{\textbf{Input:}}
\renewcommand{\algorithmicensure}{\textbf{Output:}}
\begin{algorithmic}
\REQUIRE Initial policy $\pi_\theta$, Prompts dataset $\mathcal{Q}$, Maximum number of forking points $K_{\max}$, Maximum number of branch rollouts per forking point $B_{\max}$, Base rollouts $N$, Diversity coefficient $\alpha$.
\ENSURE Optimized policy $\pi^*_\theta$

\STATE Sample prompt $q \sim \mathcal{Q}$

\STATE \textbf{Phase 1: Base Rollouts}
\STATE Generate $N$ independent base trajectories $\mathcal{R} = \{\tau^{(1)}, \dots, \tau^{(N)}\}$ via $\pi_\theta(\cdot \mid q)$
\STATE Compute $V(\mathrm{root}) = \frac{1}{N}\sum_{i=1}^N r(\tau^{(i)})$
\STATE Compute adaptive budgets: $\hat{K} = \lceil K_{\max}(1 - V(\mathrm{root})) \rceil$, \ $\hat{B} = \lceil B_{\max}(1 - V(\mathrm{root})) \rceil$

\STATE \textbf{Phase 2: Adaptive Branch Rollouts}
\STATE Initialize tree path set $\mathcal{T} \leftarrow \mathcal{R}$
\IF{$\hat{K} > 0$ \AND $\hat{B} > 0$}
    \FOR{each $\tau \in \mathcal{R}$}
        \STATE Select $\hat{K}$ forking points $\mathcal{F}$ using \texttt{sent-entropy} forking strategy (\autoref{alg:forking_selection})
        \FOR{each $f \in \mathcal{F}$}
            \STATE Extract trajectory prefix $x_f \leftarrow [q; \tau_{1:f}]$
            \STATE Generate $\hat{B}$ branch continuations from prefix $x_f$
            \STATE Add generated branches to $\mathcal{T}$
        \ENDFOR
    \ENDFOR
\ENDIF

\STATE \textbf{Phase 3: Block-level Advantage Estimation}
\STATE Partition $\mathcal{T}$ into a set of contiguous blocks $\mathcal{B}$
\FOR{each block $b \in \mathcal{B}$}
    \STATE Compute $\hat{V}_{MC}(s_{\text{start}}^{(b)})$ and $\hat{V}_{MC}(s_{\text{end}}^{(b)})$
    \STATE $\hat{A}_{\text{base}}(b) = r(b) + \hat{V}_{MC}(s_{\text{end}}^{(b)}) - \hat{V}_{MC}(s_{\text{start}}^{(b)})$
    \STATE Compute $\text{Div}_{\text{sib}}(b)$ using block embeddings (\autoref{equ:sibling_diversity})
    \STATE $\hat{A}(b) = \hat{A}_{\text{base}}(b) + \mathbb{I}(\hat{A}_{\text{base}}(b) > 0) \cdot \alpha \cdot \text{Div}_{\text{sib}}(b)$
\ENDFOR

\STATE \textbf{Phase 4: Policy Update}
\STATE Update $\theta$ by maximizing $\mathcal{J}_{\mathrm{DATPO}}(\theta)$ (\autoref{equ:objective_DATPO}) using advantages $\{\hat{A}(b)\}_{b \in \mathcal{B}}$
\STATE Anneal $\alpha$ according to schedule

\RETURN Optimized policy $\pi^*_\theta$
\end{algorithmic}
\end{algorithm}

\subsection{Comparison with Other Tree-based Methods}
\label{sec:comparison_tree_methods}

To contextualize DATPO within the landscape of recent tree-based RLVR methodologies, we highlight two fundamental distinctions between our approach and existing frameworks such as TreeRL \citep{hou2025treerl} and AttnRL \citep{liu2025attention}.

\paragraph{Structural Robustness in Optimization}
Both TreeRL and AttnRL generate tree-structured rollouts but ultimately flatten these trees into a set of $N(1+K \times B)$ independent sequences to perform sequence-level policy updates \citep{hou2025treerl}. A critical flaw in this unfolding process is that common trajectory prefixes are duplicated across multiple sequences, causing the policy to redundantly update on the same early tokens. To mitigate the resulting overfitting, these methods rely on a heuristic penalty, artificially downweighting the advantage of non-leaf steps by dividing it by the square root of the descending leaf count ($\sqrt{|L(s_n)|}$) \citep{hou2025treerl, liu2025attention}. 

In contrast, DATPO is structurally immune to this redundancy. By strictly partitioning the generated tree into non-overlapping contiguous blocks, every token within the tree belongs to exactly one block. Because DATPO performs explicit block-level updates directly over the tree topology, it naturally prevents duplicate gradient updates on shared prefixes, providing a highly robust credit assignment mechanism without the need for ad-hoc advantage penalization.

\paragraph{Motivation for Difficulty-Adaptive Exploration}
While AttnRL \citep{liu2025attention} similarly incorporates a difficulty-adaptive exploration mechanism, its approach is fundamentally driven by computational efficiency. Specifically, AttnRL reduces the sampling budget for easy problems to prevent inefficient exploration and filter out responses with zero advantages. 

Conversely, DATPO's difficulty-adaptive rollout originates from a completely different motivation: maximizing intrinsic reasoning coverage (pass@$k$). As demonstrated in \hyperref[sec:adaptive_rollout]{Section~\ref{sec:adaptive_rollout}}, adaptive budget allocation is not merely a resource-saving heuristic, but a key algorithmic factor for successfully expanding a model's pass@$k$ without inducing mode collapse. Built entirely around this objective, DATPO further maximizes coverage through the structural superiority of trees (\hyperref[sec:tree_vs_parallel]{Section~\ref{sec:tree_vs_parallel}}) and optimal semantic forking (\hyperref[sec:optimal_forking]{Section~\ref{sec:optimal_forking}}). Thus, while DATPO and AttnRL share superficial similarities in their adaptive nature, they are built upon fundamentally distinct motivations and optimization goals.

\section{Details on Main Experiments (\hyperref[sec:experiments]{Section~\ref{sec:experiments}})}
\label{app:experiments}
\subsection{Experimental Setups}
\label{app:experimental_setups}
\paragraph{Models and Datasets}
\label{par:app_models_datasets}
We use Qwen2.5-3B-Base \citep{qwen2.5} and Qwen3-4B-Base \citep{qwen3} as base models. For training, we use the MATH dataset \citep{hendrycks2021measuring}. From the full set of 12,500 problems, we exclude the 500 problems used as the MATH500 and use the remaining 12,000 problems as the training set.

\paragraph{Evaluation}
\label{par:app_evaluation_metrics}
We evaluate the trained models on mathematical reasoning benchmarks, specifically MATH500 \citep{hendrycks2021measuring}, AIME26, AIME25, AIME24, and AMC23. We report avg@$k$ as the primary evaluation metric across all benchmarks. All evaluations were performed with a temperature of 1.0 and a maximum generation length of 2048. To assess the reasoning coverage and test-time scaling performance of the trained models, we additionally report pass@$k$ and maj@$k$ (majority voting), respectively. We use $k=8$ for MATH500, and $k=64$ for the remaining benchmarks. For robustness, both pass@$k$ and maj@$k$ results are computed by averaging over three independent evaluation runs for each dataset. The final maj@$k$ in \autoref{fig:test_time_scaling} is then reported as the average of these results across all mathematical reasoning benchmarks.

\paragraph{Baselines}
\label{par:app_baselines}
We compare DATPO with several baselines: (1) \textbf{Base}, the base model without any additional training; (2) \textbf{GRPO} \citep{shao2024deepseekmath}, standard GRPO enhanced with the token-level loss and clip-higher strategy from DAPO \citep{yu2025dapo}. Following prior work \citep{liao-etal-2025-enhancing} showing that dynamic sampling can be detrimental for relatively weak reasoning models, we intentionally adopt only this specific subset of DAPO components; (3) \textbf{Dr.GRPO} \citep{liu2025understanding}, a refined variant of GRPO that corrects structural biases in advantage estimation; (4) \textbf{TreeRL} \citep{hou2025treerl}, a tree-based RL method that selects forking points from high-entropy tokens and estimates process-level advantages by combining local and global advantages; and (5) \textbf{AttnRL} \citep{liu2025attention}, another tree-based RL method built upon TreeRL with a focus on improving efficiency by selecting forking points based on attention scores and adopting adaptive sampling and one-step off-policy. For TreeRL and AttnRL, we adopt the same tree-expansion hyperparameters as reported in original papers. To ensure a fair comparison, we control the total number of generated tokens per problem to be comparable across all methods (detailed in \hyperref[app:computational_cost_wall_clock]{Appendix~\ref{app:computational_cost_wall_clock}}).

\begin{table}[t]
\centering
\begin{tabular}{lc}
\toprule
\textbf{Method} & \textbf{Avg. Generated Tokens} \\
\midrule
GRPO & 8,755.7 \\
Dr.GRPO & 8,608.3 \\
TreeRL & 11,029.9 \\
AttnRL & 8,807.3 \\
DATPO (Ours) & 8,952.4 \\
\bottomrule
\end{tabular}
\caption{Average number of generated tokens per question during training across different methods.}
\label{tab:avg_generated_tokens}
\end{table}

\subsection{Implementation Details}
\label{app:training_config}
\paragraph{Sampling and Algorithmic Hyperparameters}
Across all methods, we set the sampling temperature to 1.0 and the maximum generation length to 1024 tokens. For \textbf{GRPO} and \textbf{Dr.GRPO}, we sample 16 responses per problem. For the tree-based baselines, we strictly follow the hyperparameters reported in their respective original papers. Specifically, for \textbf{TreeRL}, we set the tree expansion parameters to $(N, K, B) = (6, 2, 2)$. For \textbf{AttnRL}, we use maximum budgets of $(N, K_{\max}, B_{\max}) = (6, 2, 2)$, select the top 20\% of FCI steps as forking candidates, and apply $\Delta = 4$ and $\lambda = 0.9$ for adaptive batching. For our proposed \textbf{DATPO}, we set the adaptive tree parameters to $(N, K_{\max}, B_{\max}) = (4, 3, 4)$. We intentionally employ a larger maximum branch count ($B_{\max} = 4$) compared to the baselines, as calculating the sibling diversity term ($\text{Div}_{\text{sib}}$) across merely two branches ($B=2$) diminishes its significance. To accommodate this wider branching while keeping the total number of generated tokens per problem comparable to the baseline methods, we adjusted the base and forking budgets accordingly. The diversity coefficient $\alpha$ is initialized to 0.2 and linearly annealed to 0 over the course of training. To compute the semantic diversity representations, we utilize the \texttt{gte-large-en-v1.5} embedding model \citep{zhang-etal-2024-mgte}. Furthermore, to prevent Out-Of-Memory (OOM) errors caused by the expansion of adaptive methods (AttnRL and DATPO), we apply rollout chunking during the branch rollout phase. We generate branch rollouts in chunk sizes of 384 for Qwen2.5-3B-Base and 64 for Qwen3-4B-Base.

\begin{table}[t]
\centering
\footnotesize
\begin{tabular}{lc}
\toprule
\textbf{Method} & \textbf{Wall Clock (Hours)} \\
\midrule

\rowcolor{gray!15}
\multicolumn{2}{c}{\textbf{Parallel Rollout}} \\
GRPO & 26.3 \\
Dr.GRPO & 27.1 \\

\midrule
\rowcolor{gray!15}
\multicolumn{2}{c}{\textbf{Tree-based Rollout}} \\
TreeRL & 56.2 \\
AttnRL & 55.5 \\
DATPO w/o diversity & 56.0 \\
DATPO & 60.3 \\
DATPO w/ one-step off-policy & 60.0 \\

\bottomrule
\end{tabular}
\caption{Wall-clock comparison across different methods.}
\label{tab:main_wall_clock}
\end{table}

\paragraph{Training Hyperparameters}
We formulate the task with a binary verifiable reward, assigning a reward of 1 if the final answer is correct and 0 otherwise. Models are optimized using the AdamW optimizer with $\beta = (0.9, 0.999)$ and a constant learning rate of $5 \times 10^{-6}$. We omit the KL divergence penalty and the clipping ratio is set to $0.2$. For the GRPO baseline enhanced with the clip-higher strategy, we set the upper clip ratio to 0.28 and the lower to 0.2. During training, we update the policy using 16 prompts per step, maintaining a global mini-batch size of 64 and a micro-batch size of 2 for gradient accumulation across all methods. The total effective batch size varies by method: GRPO and Dr.GRPO utilize a fixed batch size of 256; TreeRL employs a batch size of 480 (comprising 96 base rollouts and 384 branch rollouts); whereas the batch sizes for AttnRL and DATPO dynamically adjust per step due to their adaptive nature. Notably, because DATPO performs block-level policy updates, its mini-batch and micro-batch sizes are defined with respect to the number of contiguous blocks rather than full sequences. All experiments were executed on a computing cluster equipped with 8 NVIDIA A100 GPUs, where each individual training run was conducted on a single A100 GPU. Our codebase is developed based on the GRPO-Zero framework\footnote{\url{https://github.com/policy-gradient/GRPO-Zero}}.

\subsection{Computational Cost and Wall-Clock Time}
\label{app:computational_cost_wall_clock}
To ensure a fair comparison, we control the total number of generated tokens per problem to be comparable across all methods. The actual average number of generated tokens per question during training for each method is detailed in \autoref{tab:avg_generated_tokens}. However, despite this token-level equivalence, tree-based methods generally incur a higher wall-clock time compared to standard parallel rollouts, as reported in \autoref{tab:main_wall_clock}. This discrepancy arises from the inherent sequential dependency in tree generation: branch rollouts cannot commence until the base rollouts are fully generated and forking points are explicitly selected, creating a computational bottleneck. Furthermore, our proposed DATPO requires additional overhead to compute block-level embeddings for the diversity term, slightly increasing its wall-clock time relative to other tree-based baselines. We note that adopting the one-step off-policy approach proposed in AttnRL \citep{liu2025attention} can marginally mitigate this latency. \autoref{tab:runtime_breakdown} provides a detailed breakdown of the per-step runtime of DATPO by computational component.

\begin{table}[t]
    \centering
    \small
    \begin{tabular}{lcc}
        \toprule
        \textbf{Component} & \textbf{Time (s)} & \textbf{Total runtime (\%)} \\
        \midrule
        Policy              & 150.44 & 50.14 \\
        Diversity embedding & 13.65  & 4.55  \\
        Update              & 102.17 & 34.05 \\
        Others              & 33.76  & 11.25 \\
        \midrule
        \textbf{Total}      & \textbf{300.02} & \textbf{100.00} \\
        \bottomrule
    \end{tabular}
    \caption{Per-step runtime breakdown of DATPO across different computational components.}
    \label{tab:runtime_breakdown}
\end{table}

\begin{figure}[t]
    \centering
    \includegraphics[width=\linewidth]{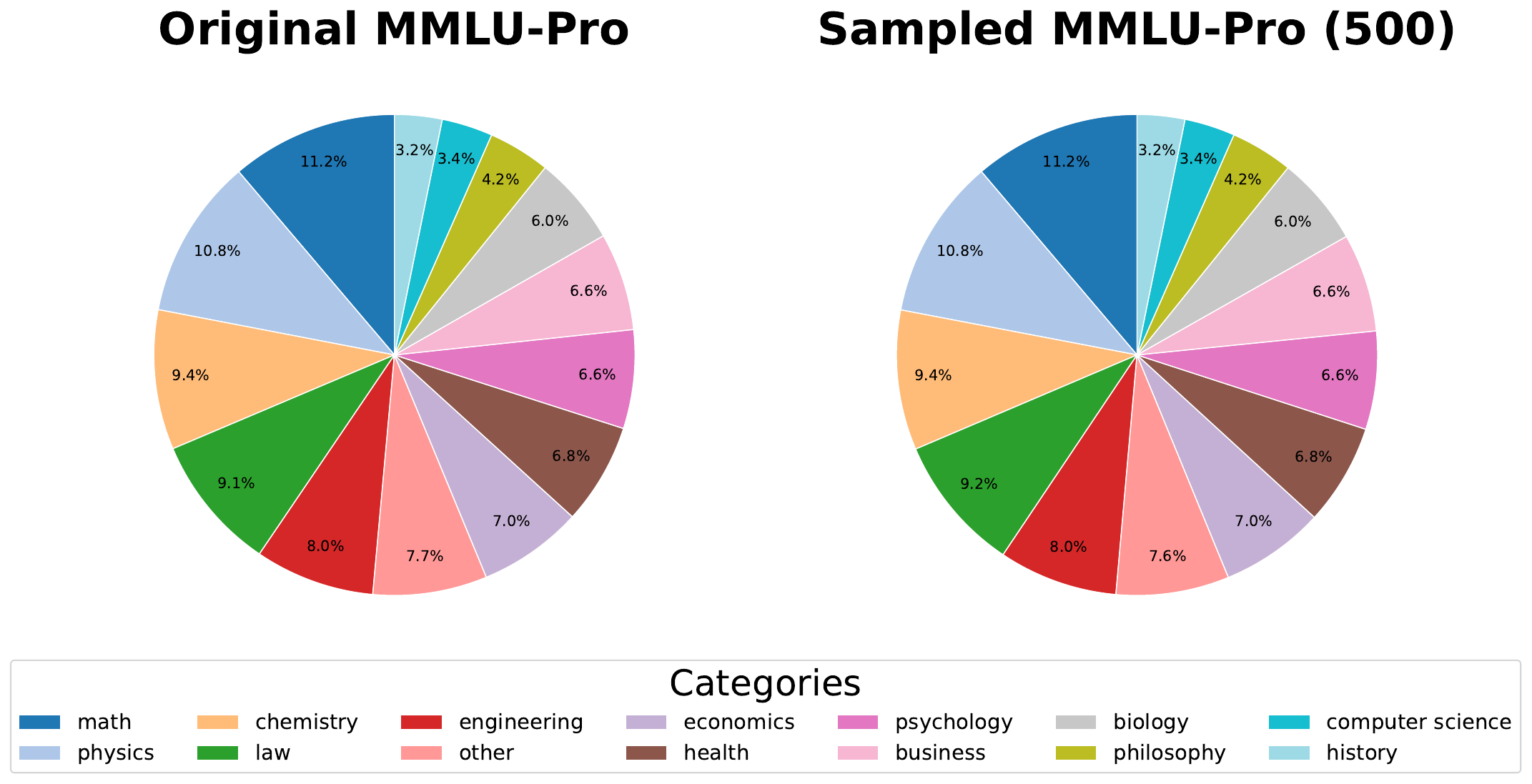}
    \caption{Comparison of category-wise distributions between the original MMLU-Pro dataset and our sampled subset (500 instances). Stratified sampling ensures the original proportions are preserved.}
    \label{fig:mmlu_pro_dist}
\end{figure}

\subsection{Generalization to Out-Of-Domain Tasks}
\label{app:ood_generalization}
To investigate whether the models trained on mathematical reasoning datasets can lead to improvements in reasoning performance across other domains, we additionally conduct out-of-domain evaluations using GPQA-Diamond \citep{rein2023gpqa} and MMLU-Pro \citep{wang2024mmlu}. For MMLU-Pro, we applied stratified sampling to evaluate on a subset of 500 instances, preserving the original proportional distribution across categories such as math, physics, chemistry, engineering, law, and economics. A comparison of the category-wise distributions between the original and sampled datasets is provided in \autoref{fig:mmlu_pro_dist}. We evaluate the models trained with each method using two base models, Qwen2.5-3B-Base and Qwen3-4B-Base, and report the avg@8 metric for both benchmarks.

\autoref{tab:ood_results} presents the results of the out-of-domain evaluation. Overall, DATPO demonstrates consistent performance across both the GPQA-Diamond and MMLU-Pro datasets. Specifically, compared to the strongest baselines, DATPO achieves a +2.4 percentage point improvement on GPQA-Diamond and a marginal -0.1 point difference on MMLU-Pro using the Qwen2.5-3B-Base model. When applied to the Qwen3-4B-Base model, it yields consistent improvements of +1.2 and +0.4 percentage points on the respective benchmarks. These results demonstrate that the enhancements in reasoning capacity observed on in-domain mathematical benchmarks successfully generalize to improved reasoning performance on out-of-domain tasks.

\begin{table}[t]
\centering
\small
\setlength{\tabcolsep}{6pt}
\renewcommand{\arraystretch}{1}

\begin{tabular}{lcc}
\toprule

\textbf{Method} & \textbf{GPQA-Diamond} & \textbf{MMLU-Pro} \\

\midrule

\rowcolor{gray!15}
\multicolumn{3}{c}{\textbf{Qwen2.5-3B-Base}} \\
Base          & 21.1 & 18.4 \\
GRPO          & 25.3 & 32.6 \\
Dr.GRPO       & \underline{26.1} & 31.2 \\
TreeRL        & 25.3 & \underline{33.1} \\
AttnRL        & 25.9 & \textbf{33.2} \\
DATPO (Ours)  & \textbf{28.5} & \underline{33.1} \\

\midrule

\rowcolor{gray!15}
\multicolumn{3}{c}{\textbf{Qwen3-4B-Base}} \\
Base          & 24.2 & 34.8 \\
GRPO          & 35.9 & 55.9 \\
Dr.GRPO       & 38.2 & \underline{57.3} \\
TreeRL        & 36.7 & 55.6 \\
AttnRL        & \underline{38.4} & 56.6 \\
DATPO (Ours)  & \textbf{39.6} & \textbf{57.7} \\

\bottomrule
\end{tabular}

\caption{Out-of-domain evaluation results. We report avg@8 accuracy for each dataset.}
\label{tab:ood_results}
\end{table}

\subsection{Further Ablation Studies}
\label{app:further_ablation}
\paragraph{Effect of Difficulty-Adaptive Rollout}
While \hyperref[sec:main_results]{Section~\ref{sec:main_results}} indirectly demonstrates the benefits of our approach against non-adaptive baselines (e.g., TreeRL), we conduct this ablation to explicitly isolate the impact of the difficulty-adaptive mechanism by evaluating a static, non-adaptive variant of DATPO. Applying our default tree-expansion parameters of $(N, K_{\max}, B_{\max}) = (4, 3, 4)$ to a static rollout would significantly increase the overall training budget. Therefore, to maintain a token consumption comparable to our adaptive framework, we restrict the tree-expansion parameters of this non-adaptive variant to a narrower configuration of $(N, K, B) = (6, 2, 2)$.

As reported in \autoref{tab:ablation_adaptive}, the full DATPO framework outperforms the non-adaptive variant in both avg@$k$ and pass@$k$. The lower pass@$k$ of the non-adaptive model aligns with our empirical analysis in \hyperref[sec:adaptive_rollout]{Section~\ref{sec:adaptive_rollout}}: uniformly expanding the rollout budget across all problems, regardless of their difficulty, can lead to a decrease in pass@$k$. Furthermore, the absence of difficulty adaptation also degrades avg@$k$. We attribute this to a compounding effect during training: the restricted reasoning coverage artificially limits the variety of valid reasoning paths discovered, which consequently undermines the effectiveness of the annealed sibling-diversity term that relies on a rich, diverse set of trajectories to properly optimize the policy.

\begin{table}[t]
\centering
\small
\setlength{\tabcolsep}{6pt}
\renewcommand{\arraystretch}{1.2}
\begin{tabular}{lcc}
\toprule
\multirow{2}{*}{\textbf{Benchmark}}
& \textbf{DATPO w/o}
& \multirow{2}{*}{\textbf{DATPO}} \\
& \textbf{difficulty-adaptive} & \\
\midrule
MATH500 & 62.9 / 81.0 & 63.5 / 81.7 \\
AIME26  & 2.6 / 21.1  & 2.4 / 33.3 \\
AIME25  & 1.7 / 25.6  & 1.6 / 33.3 \\
AIME24  & 5.4 / 31.1  & 4.8 / 35.6 \\
AMC23   & 37.0 / 90.0 & 39.8 / 90.8 \\
\midrule
\textbf{Average} & 21.9 / 49.8 & \textbf{22.4 / 54.9} \\
\bottomrule
\end{tabular}
\caption{Ablation study on difficulty-adaptive rollout. Results are reported as avg@8 / pass@8 for MATH500 and avg@64 / pass@64 for all other benchmarks.}
\label{tab:ablation_adaptive}
\end{table}

\begin{figure}[t]
    \centering
    \includegraphics[width=\columnwidth]{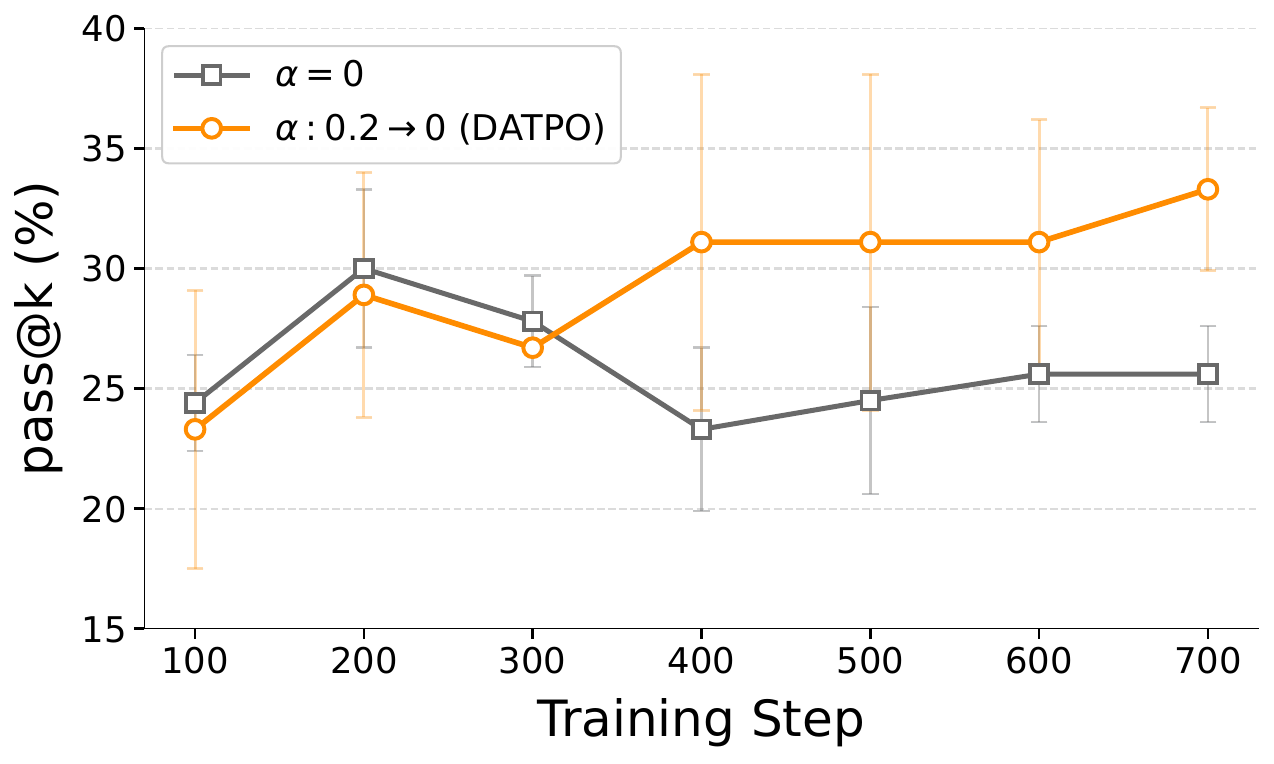}
    \caption{pass@64 on AIME26 across training checkpoints under different diversity-coefficient schedules.}
    \label{fig:diversity-training-dynamics}
\end{figure}

\paragraph{Training Dynamics of the Diversity-Augmented Advantage.}
\hyperref[tab:ablation_alpha]{Table~\ref{tab:ablation_alpha}} reports the effect of the diversity coefficient only at the final checkpoint. To examine how the diversity-augmented advantage affects reasoning coverage throughout training, we additionally evaluate pass@64 on AIME26 at 100-step intervals. We compare training without the diversity term ($\alpha=0$) against DATPO with the default linearly annealed schedule ($\alpha: 0.2 \rightarrow 0$). Each result is averaged over three independent evaluation runs and reported as the mean $\pm$ standard deviation.

As shown in \hyperref[fig:diversity-training-dynamics]{Figure~\ref{fig:diversity-training-dynamics}}, without the diversity term, pass@64 peaks at $30.0\%$ at step 200 and decreases to $25.6\%$ by step 700. In contrast, DATPO exhibits an overall upward trend after the initial training stage and reaches $33.3\%$ at step 700. These results suggest that the diversity-augmented advantage mitigates the early saturation of reasoning coverage and promotes its continued expansion during training.

\paragraph{Effect of Embedding Models}
As noted in \hyperref[app:training_config]{Appendix~\ref{app:training_config}}, DATPO defaults to using \texttt{gte-large-en-v1.5} to compute the block-level sibling-diversity term ($\text{Div}_{\text{sib}}$). To assess the sensitivity of our framework to the choice of the embedding model, we conduct an additional ablation study comparing it against three alternative models of varying scales: \texttt{all-mpnet-base-v2} \citep{reimers-gurevych-2019-sentence}, \texttt{bge-m3} \citep{chen-etal-2024-m3}, and \texttt{gte-base-en-v1.5} \citep{zhang-etal-2024-mgte}.

As reported in \autoref{tab:ablation_embedding_models}, while the 434M-parameter \texttt{gte-large-en-v1.5} model achieves the highest \text{avg@8} accuracy (63.5\%), smaller architectures such as \texttt{all-mpnet-base-v2} (110M) and \texttt{gte-base-en-v1.5} (137M) also yield highly competitive performance, even marginally outperforming the default configuration in \text{pass@8} (81.9\%). These results demonstrate that DATPO is robust to the scale and specific choice of the underlying embedding model. Consequently, the effectiveness of our diversity-augmented advantage appears to stem primarily from capturing stable, relative semantic distance signals among sibling blocks, rather than relying strictly on the absolute capacity or size of the embedding model itself.

\begin{table}[t]
\centering
\small
\setlength{\tabcolsep}{8pt}
\renewcommand{\arraystretch}{1.2}
\begin{tabular}{lccc}
\toprule
\multirow{2}{*}{\textbf{Embedding}} & \multirow{2}{*}{\textbf{Size}} & \multicolumn{2}{c}{\textbf{MATH500}} \\
\cmidrule(lr){3-4}
 &  & \textbf{avg@8} & \textbf{pass@8} \\
\midrule
\texttt{all-mpnet-base-v2} & 110M & 62.9 & 81.9 \\
\texttt{bge-m3}            & 560M & 61.6 & 81.7 \\
\texttt{gte-base-en-v1.5}  & 137M & 63.4 & 81.9 \\
\texttt{gte-large-en-v1.5} & 434M & 63.5 & 81.7 \\
\bottomrule
\end{tabular}
\caption{Ablation study on different embedding models.}
\label{tab:ablation_embedding_models}
\end{table}

\begin{table}[t]
\centering
\small
\setlength{\tabcolsep}{8pt}
\renewcommand{\arraystretch}{1.4}
\begin{tabular}{lccc}
\toprule
\multirow{2}{*}{\textbf{Method}} & \multirow{2}{*}{\textbf{Avg. Tokens}} & \multicolumn{2}{c}{\textbf{MATH500}} \\
\cmidrule(lr){3-4}
 & & \textbf{avg@8} & \textbf{pass@8} \\
\midrule
TreeRL (4,3,4) & 17,766.0 & 62.5 & 81.3 \\
AttnRL (4,3,4) & 12,478.7 & 62.1 & 80.0 \\
DATPO & 8,952.4 & 63.5 & 81.7 \\
\bottomrule
\end{tabular}
\caption{Performance of baselines under wider tree topology. Avg. Tokens denotes the average number of generated tokens per question during training. The values in parentheses specify the tree expansion hyperparameters $(N, K, B)$.}
\label{tab:baseline_topology}
\end{table}

\subsection{Performance of Baselines under Wider Tree Topology}
\label{app:baseline_wider_topology}
As detailed in \hyperref[app:training_config]{Appendix~\ref{app:training_config}}, the tree-based baseline methods (TreeRL and AttnRL) were evaluated using the default tree expansion hyperparameters reported in their original papers, specifically $(N, K, B) = (6, 2, 2)$ for TreeRL and $(N, K_{\max}, B_{\max}) = (6, 2, 2)$ for AttnRL. In contrast, our proposed DATPO utilized a wider branching configuration of $(N, K_{\max}, B_{\max}) = (4, 3, 4)$ to ensure a sufficient number of branch rollouts for computing the sibling-diversity term. Although we rigorously constrained the total number of generated tokens to be comparable across all methods in our main experiments (\hyperref[app:computational_cost_wall_clock]{Appendix~\ref{app:computational_cost_wall_clock}}), one might argue that DATPO's performance gains simply stem from this wider tree topology ($B=4$) rather than its algorithmic design. 

To isolate the impact of our algorithmic design, we conduct an additional experiment where both TreeRL and AttnRL are trained using the identical wider configuration of $(4, 3, 4)$ (or $(4, 3, 4)$ maximum limits for AttnRL).

As reported in \autoref{tab:baseline_topology}, forcing the baselines into this wider topology does not improve their overall performance. We observe three key findings: First, our proposed DATPO achieves the highest performance in both avg@8 and pass@8 while generating the fewest average tokens per problem during training. Second, compared to their original $(6, 2, 2)$ configurations (as reported in \autoref{tab:main_results}), both TreeRL and AttnRL experience performance degradation in some metrics when scaled to the $(4, 3, 4)$ topology. Taken together, these results confirm that the superiority of our framework is driven by the algorithmic improvements of the difficulty-adaptive rollout and diversity-guided exploration, rather than merely relying on wider tree expansion hyperparameters.

\section{Further Related Works}
\subsection{Tree-based Search in Reinforcement Learning}

Tree-based search, famously successful in reinforcement learning milestones like AlphaGo \citep{silver2016mastering}, is now being adapted to the generative landscape of LLMs to unlock more systematic exploration. TreeRL \citep{hou2025treerl} pioneers this with an on-policy framework using entropy-guided branching. To address computational bottlenecks, AttnRL \citep{liu2025attention} improves efficiency by leveraging difficulty-aware adaptive sampling within a one-step off-policy pipeline. Recent research optimizes structural efficiency: TEMPO \citep{tran2025exploiting} leverages prefix-tree structure for branch-aware credit assignment, while TreePO \citep{li2025treepo} maximizes KV-cache reuse. To prevent paths from collapsing into homogeneous reasoning, LATR \citep{xing2025lookahead} explicitly maximizes trajectory-level diversity by dynamically branching and pruning semantically redundant paths via lookahead simulation. Furthermore, TreeAdv \citep{cao2026treeadv} refines credit assignment by combining an entropy-based branching method with a tree-structured advantage redistribution. While these methods primarily utilize tree structures for credit assignment or computational efficiency, our work leverages tree structures with the primary goal of expanding the model's reasoning coverage.

\section{Use of AI assistants}

AI assistants were utilized for code implementation and for refining the linguistic presentation of the manuscript. All AI-generated outputs were carefully reviewed, modified, and validated by the authors.

\end{document}